\documentclass[12pt]{article}

\usepackage{sibarticle}
\usepackage{wrapfig, ulem, graphics}
\usepackage{mathtools}
\usepackage{comment}

\graphicspath{{figures/}}

\newcommand{\CG}{\mathcal{G}}
\newcommand{\CL}{\mathcal{L}}
\newcommand{\RR}{\mathbb{R}}
\newcommand{\bm}[1]{\mbox{\boldmath{$#1$}}}

\newcommand{\vv}{\bm{v}}
\newcommand{\vu}{\bm{u}}
\newcommand{\vr}{\bm{r}}
\newcommand{\vx}{\bm{x}}
\newcommand{\vy}{\bm{y}}
\newcommand{\vz}{\bm{z}}
\newcommand{\vbeta}{\bm{\beta}}
\newcommand{\mV}{\bm{V}}
\newcommand{\mX}{\bm{X}}
\newcommand{\mZ}{\bm{Z}}
\newcommand{\tr}{^{\intercal}}
\newcommand{\inv}{^{-1}}
\newcommand{\LESS}{\texttt{LESS}}
\newcommand{\LESSA}{\texttt{LESS-A}}
\newcommand{\LESSB}{\texttt{LESS-B}}

\title{Learning with Subset Stacking}

\ShortTitle{LESS}
\ShortAuthors{Birbil, Yıldırım, Çopur \& Akyüz}
\NumberOfAuthors{4}

\FirstAuthor{Ş. İlker Birbil\thanks{Corresponding author. Email: \href{mailto:s.i.birbil@uva.nl}{s.i.birbil@uva.nl}}}
\FirstAuthorAddress{Amsterdam Business School, University of Amsterdam, 1018 TV Amsterdam, The Netherlands}
\SecondAuthor{Sinan Yıldırım}
\SecondAuthorAddress{Faculty of Engineering and Natural Sciences, Sabancı
University, 34956 Istanbul, Türkiye}
\ThirdAuthor{Samet Çopur}
\ThirdAuthorAddress{SunExpress, Antalya, 07230, Türkiye}
\FourthAuthor{M. Hakan Akyüz}
\FourthAuthorAddress{Department of Industrial Engineering, Middle East Technical University, 06800 Ankara, Türkiye}


\allowdisplaybreaks

\begin{document}

\maketitle

\begin{abstract}
We propose a new regression algorithm that learns from a set of input-output pairs. Our algorithm is designed for populations where the relation between the input variables and the output variable exhibits a heterogeneous behavior across the predictor space. The algorithm starts with generating subsets that are concentrated around random points in the input space. This is followed by training a local predictor for each subset. Those predictors are then combined in a novel way to yield an overall predictor. We call this algorithm ``LEarning with Subset Stacking'' or \LESS{}, due to its resemblance to the method of stacking regressors. We offer bagging and boosting variants of \LESS{} and test against the state-of-the-art methods on several datasets. Our comparison shows that \LESS{} is highly competitive.
\end{abstract}

\section{Introduction}\label{sec: Introduction}

This paper concerns the application of subset sampling and aggregation to achieve effective learners with computational efficiency on large-scale data. The general approach is to partition the dataset into (random) subsets on which local predictors are trained. Then, the prediction of test data is achieved by using an aggregation over the local predictors.

Learning with subsets obtained from a given training set has been quite popular due to its potential to improve generalization performance. Such an approach is called bagging when multiple predictors are treated the same, and their average is used in the aggregation stage as the overall predictor \citep{Breiman_1996a}. Bagging uses sampling with replacement to construct training datasets (i.e., subsets) of the same size as the original dataset, similar to the bootstrapping technique. To achieve savings from a computational perspective as well as further improvements in generalization performance, bagging has been applied to randomly drawn training subsets; such methods are also known as subsample aggregating, shortly called subagging \citep{Buhlmann_and_Yu_2002, Andonova_et_al_2002, Evgeniou_et_al_2004}. \cite{Wolpert1992StackedGeneralization} and \cite{Breiman_1996b} propose the so-called stacking as an aggregation method. Stacking optimizes the weights of the individual predictors used at the aggregation level and generally provides better prediction accuracy than a single estimator. Later, \cite{BuhlmannMeinshausen2015Magging} offers a maximin aggregation approach, called magging, which inherits properties from bagging and stacking. In particular, magging uses a weighted average of the estimators, where subsets are constructed similarly to bagging, and estimator weights are optimized to solve a different optimization problem than the one for stacking. 

Although the bagging, stacking, and magging methods all successfully reduce the mean squared error of unstable training mechanisms via aggregation, they lack local proximity information of the samples with respect to training subsets. In this study, we offer a novel method that integrates the above ideas within a learning algorithm by using the proximity information of data points to localized subsets. Since our method works with local subsets of samples and resembles stacking by combining multiple learners in a global learner, we call our method ``LEarning with Subset Stacking'' or, in short, \LESS{}. As \LESS{} learns from the predictions of individual learners on the subsets, it can be considered a meta-learning algorithm. 

The following outlines how \LESS{} works. First, multiple subsets of the training dataset are formed. 
After that, a local predictor is trained over each subset. Following the local training step, all training samples are then transformed into a modified feature space. The modified feature for a sample is a vector of weighted predictions obtained for that training sample by the local predictors. This weighting is carried out in such a way that the local predictions for that sample are assigned importance based on the distance from their corresponding subsets. 
At the global level, an aggregate prediction is performed on the transformed feature space and the responses. As can be deduced from this description, \LESS{} is very generic and can be applied to most regression problems. The ultimate prediction model depends on the local learners, the distance-based scaling function, and the global learner.  See Figure \ref{fig:1level} for a schematic view of \LESS{}. 

\LESS{} has similarities with bagging, stacking, and magging, as well as the existing local learning or aggregation methods. Nevertheless, the mechanism for the selection of subsets and the adopted approach for global learning are different from each of those methods. \LESS{} can also be considered a local learning method. There are several local learning methods in the literature; a non-exhaustive list contains kernel regression \citep{Nadaraya_1964}, the locally estimated scatterplot smoothing \citep{Cleveland_1979}, Bayesian committee machine \citep{Tresp_2000, Liu_et_al_2018}, Gaussian process regression \citep{Rasmussen_and_Williams_2006}, mixture-of-experts methods \citep{Rasmussen_and_Ghahramani_2001, Meeds_and_Osindero_2005, Shahbaba_and_Neal_2009, Parelta_and_Soto_2014}, and locally linear ensemble for regression \citep{Kang_and_Kang_2018}. \LESS{} differs from each one of them in the way it creates the data subsets and combines the local predictions. By constructing a feature vector for each data point from the local predictions, \LESS{} is also related to feature learning methods \citep{Bengio_et_al_2013, Le-Khac_et_al_2020}. Parallelisms can also be made between \LESS{} and Fuzzy Inference Systems (FIS); see \textit{e.g.}, \cite{Mamdani_1974, Takagi_Sugeno_1985_FIS} for the foundations and \cite{Babuska_1998} for a book review, and \cite{Shihabudheenetal2018Review} for a survey on fuzzy inference systems. Especially, the adaptive neuro-fuzzy generalization of FIS (ANFIS), has been a popular framework for modeling non-homogeneous input-output relations \citep{Jang1993ANFIS, Zhang_2024}. The relation of ANFIS to \LESS{} is due to the existence of subsets and the way the final output is formed as a combination of the outputs of those subsets, in both models. However, \LESS{} and  ANFIS differ on a methodological level, especially in terms of the approach in the design and combination of their subsets. \LESS{} can also be transformed into a Takagi-Sugeno (TS) type model when both local and global learning steps are achieved by linear or convex programs. Nevertheless, \LESS{} brings different perspectives to subset creation, local and global learning as well as interpretability of the results. We discuss \LESS{} in relation to other methods in detail in \ref{sec: Related Methods}.

Our computational study on a set of regression problems has shown that \LESS{} is quite competitive among many well-known learning methods, including those mentioned above. 
Moreover, LESS lends itself to a parallel implementation in a straightforward manner thanks to its use of multiple subsets of the samples.


\section{LEarning with Subset Stacking} \label{sec: LEarning with Subset Stacking}

In this section, we first present the methodology behind \LESS{} generically. Then, we give an example of \LESS{} with some specific choices and discuss the effects of changing its hyperparameters. After discussing the impact of weighting on \LESS{}, we introduce two possible modifications and obtain two \LESS{} variants.

\subsection{General Framework} \label{sec: General Framework}
Suppose that we have the training dataset $\{(\vx_i, y_i) : i=1,\dots,n\}$, where $\vx_i \in \RR^p$, for some $p \geq 1$, stands for the input vector, and $y_{i} \in \mathbb{R}$ denotes the scalar output value for the data point $i \in \{1, \dots, n\}$. To simplify our exposition, we shall denote the dataset by $(\mX, \vy)$. Specifically, $\mX$ is an $n \times p$ matrix whose row $i$ is $\vx_{i}\tr$, and $\vy$ is the $n \times 1$ vector of outputs $y_{1}, \ldots, y_{n}$.

Figure \ref{fig:1level} shows a depiction of \LESS{}. In the subsequent part, we describe the steps of \LESS{} in the following order: subset selection and local learning, feature generation, global training, and averaging or boosting.

\begin{figure*}[t]
  \centering
  \includegraphics[width=\textwidth]{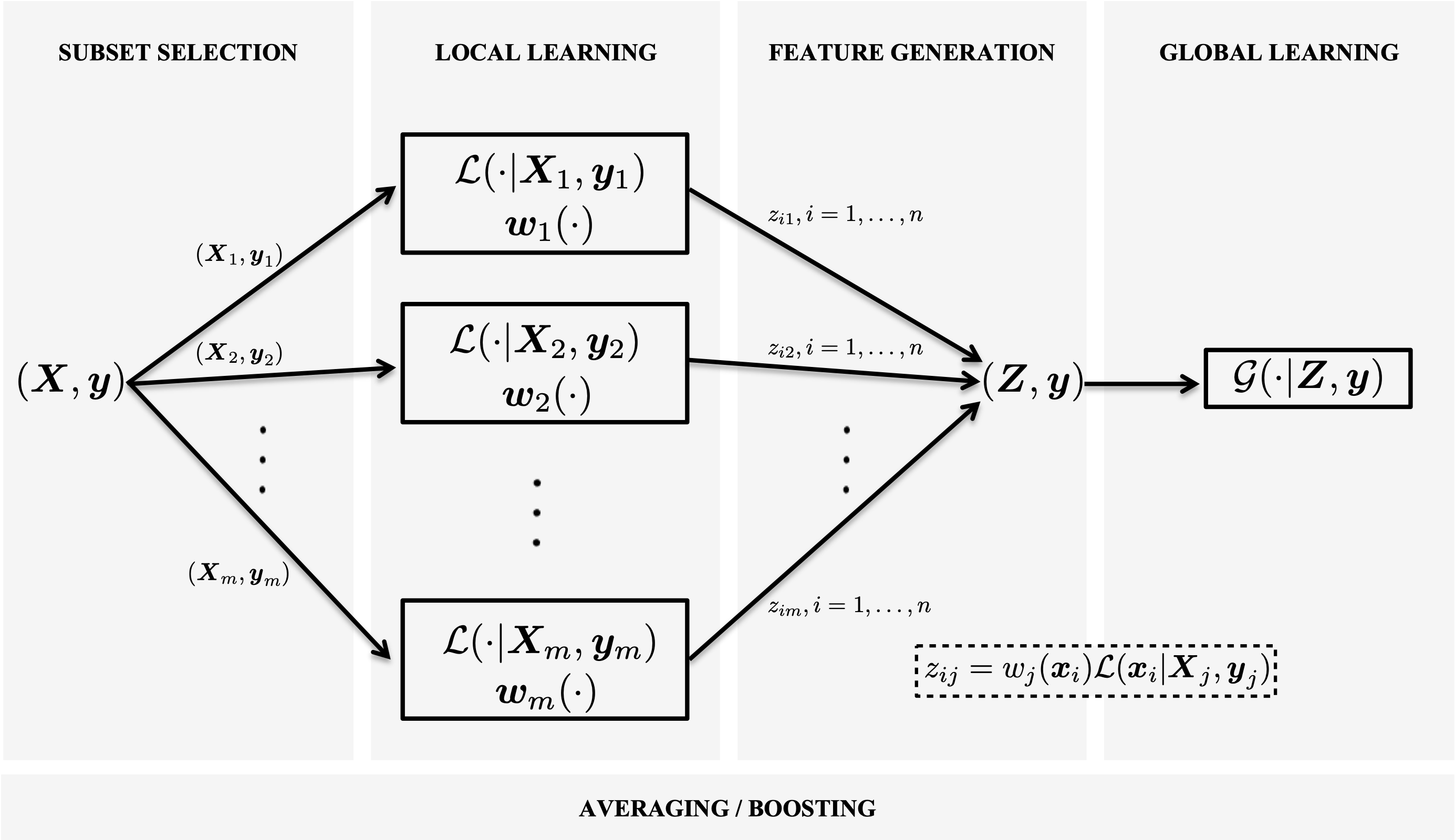}
  \caption{Schematic view of the basic \LESS{} algorithm.}
  \label{fig:1level}
\end{figure*}

{\noindent}\textbf{Subset Selection and Local Learning.} 
The first step of our learning algorithm is to generate a collection of subsets $(\mX_j, \vy_j) \subseteq (\mX, \vy)$, $j=1, \dots, m$. We do not restrict to partitioning the dataset into mutually exclusive, collectively exhaustive subsets. Different subsets may have common samples, or there may exist a sample that does not belong to any one of the subsets. The core idea of our approach is to train a different model with the data points in each subset. That is, for a given input point $\vx$, we obtain its prediction by training a learning method \textit{locally} with the subset $(\mX_j, \vy_j)$. We denote this local prediction by $\CL(\vx | \mX_j, \vy_j)$. 

It is worth noting that the method requires the selected subsets to be localized. Localness is necessary for the trained model to capture the input-output relation in the locality of the center of the subset. To ensure localness, we first randomly select $m$ input points from $\mX$ as \textit{anchors}, and then for each anchor, we construct a subset by selecting $k$-nearest samples to the anchor. 
Clearly, the performance of this simple method depends on the random location of the subsets $(\mX_j, \vy_j)$, $j=1, \dots, m$. While this method is straightforward to implement, a principled alternative could be using clustering techniques, such as $k$-means or spectral clustering. We shall discuss such variants in \ref{sec: Variants}.

Note that each local model is trained independently of the others. Therefore, the local models can be trained in a trivially parallel fashion. Moreover, each local model works with only a subset of data samples; hence, it is faster to train each one of them. 

{\noindent}\textbf{Feature Generation.} 
One expects that prediction at an input point $\vx$ is more reliable when a sample is \textit{closer} to set $\mX_j$. Thus, we also define a weighting function $w_{j}(\vx)$ that is inversely related to the distance between the sample $\vx$ and the subset $\mX_{j}$. Then, for each input $\vx_{i}$ in the dataset, we generate a \textit{feature vector} of weighted predictions:
\begin{equation} \label{eqn:xhat}
\vz(\vx_{i}) = [w_{1}(\vx_{i}) \CL(\vx_i |  \mX_1, \vy_1), \dots, w_{m}(\vx_{i}) \CL(\vx_i |  \mX_m, \vy_m)] \tr.
\end{equation}
To simplify our notation in the subsequent part, we will refer to $\vz(\vx_{i})$ as vector $\vz_i$ with the components $z_{ij} = w_{j}(\vx_{i}) \CL(\vx_i |  \mX_j, \vy_j)$ for $j=1, \dots, m$.

The weights $w_{j}(\vx)$ are based on a distance metric $d(\cdot, \cdot): \mathbb{R}^{p} \times  \mathbb{R}^{p} \mapsto [0, \infty)$ such that $w_{j}(\vx)$ is inversely related to the distance $d(\vx, \bar{\vx}_{j})$, where we define $\bar{\vx}_{j}$ to be the centroid of subset $j$. One way to choose the weights, which we also adopt in this work, can be
\begin{equation} \label{eq: distance-based weights}
w_{j}(\vx) = \frac{\exp\left\{-\lambda d(\vx, \bar{\vx}_{j}) \right\}}{\sum_{j' = 1}^{m} \exp\left\{-\lambda d(\vx, \bar{\vx}_{j'}) \right\} }
\end{equation}
for some $\lambda \geq 0$. We note that $\bar{\vx}_{j}$ can also be any representative point in the subset $\mX_j$, other than its centroid, that could be used to determine the distance. The parameter $\lambda$ determines the extent to which the distance factor should be taken into account. If $\lambda$ is small, the predictor weights are closer to each other; if $\lambda$ is large, then the closer predictors are given more weights than the farther ones. In Section \ref{sec: Effect of Weighting} we give a discussion on two extreme choices of $\lambda$. The parameter $\lambda$ can be tuned by using the training data, \textit{e.g.}, with cross-validation.

{\noindent}\textbf{Global Learning.} 
The next step is to train a \textit{global} learning method that results in a mapping from the generated feature vector $\vz_{i}$ to the output value $y_{i}$ for $i = 1, \ldots, n$. Let $\mZ$ be the $n \times m$ matrix, whose row $i$ is the feature vector $\vz_{i}\tr$ for input $i$. The output of this global method, when fed by a feature vector $\vz$, can be denoted as $\CG(\vz | \mZ, \vy)$. After \LESS{} is trained, the prediction for a test point $\vx_0 \in \RR^p$ simply becomes
\[
\hat{y}_0 = \CG(\vz_0 | \mZ, \vy), \quad \text{with} \quad z_{0j} = w_{j}(\vx_0) \CL(\vx_0 | \mX_{j}, \vy_{j}), \quad j = 1, \ldots, m.
\]


{\noindent}\textbf{Averaging and Boosting.} 
When the subsets are chosen randomly, \textit{e.g.}, via sampling anchor points, \LESS{} can also benefit from averaging over repetitions to reduce the variance. Averaging is a typical approach applied by ensemble methods. For \LESS{}, this simply corresponds to applying the procedure described above $b$ times and taking the average. If we denote the replication $\ell$ with a superscript, then the overall prediction for a test point $\vx_0 \in \RR^p$ with averaging leads to
\[
\hat{y}_0 = \frac{1}{b} \sum_{\ell = 1}^{b} \CG(\vz^{(\ell)}_0 | \mZ^{(\ell)}, \vy),  \text{ with } z^{(\ell)}_{0j} = w_{j}^{(\ell)}(\vx_0)\mathcal{L}(\vx_0 | \mX_{j}^{ (\ell)}, \vy_{j}^{(\ell)}); \, j = 1, \ldots, m.
\]

Boosting works slightly differently from averaging and benefits from sequentially fitting a global model on the subsequent residuals. For \LESS{}, this simply corresponds to applying the procedure described above $b$ times on the residuals, denoted with $\vr^{(\ell)}$ for $\ell = 1, \ldots, b$ with $\vr^{(1)} = \vy$ and superscript $\ell$ again standing for the replication. The prediction vectors in the sequence are accumulated over the replications and depend recursively on the previous residuals multiplied by a learning rate. That is, for each input $\vx_i$, we have 
\[
r_i^{(\ell+1)} = r_i^{(\ell)} - \rho \, \CG( \vz_i^{(\ell)} | \mZ^{(\ell)}, \vr^{(\ell)}), \ \  \ell = 1, \dots, b-1,
\]
where $\mZ^{(\ell)}$ is an $n \times m$ matrix whose $i$th row is $\vz_{i}^{(\ell) \top}$, and the latter consists of the components $z_{ij}^{(\ell)} = w_{j}(\vx_{i}) \CL(\vx_i |  \mX_j^{ (\ell)}, \vr_j^{(\ell)})$ for $j=1, \dots, m$. The parameter $\rho > 0$ is the learning rate, and $r_1^{(\ell)}, \dots, r_n^{(\ell)}$ stand for the components of the residual vector $\vr^{(\ell)}$. After this sequential training, the overall prediction for a test point $\vx_0 \in \RR^p$ becomes
\[
\hat{y}_0 = \CG(\vz^{(1)}_0 | \mZ^{(1)}, \vy) + \rho \sum_{\ell = 2}^{b} \CG(\vz^{(\ell)}_0 | \mZ^{(\ell)}, \vr^{(\ell)}),
\]
where we again use the previous notation $z^{(\ell)}_{0j} = w_{j}^{(\ell)}(\vx_0)\mathcal{L}(\vx_0 | \mX_{j}^{ (\ell)}, \vr_{j}^{(\ell)})$ for $j = 1, \ldots, m$. Observe that this is a simplified version of the boosting idea. Here, the gradient information is not elaborately used (\textit{e.g.}, no second-order information or other loss functions). Instead, \LESS{} consecutively fits global models on the current residuals, aiming to minimize the final prediction mean squared error. Therefore, we stick to the term boosting to distinguish this \LESS{} variant from the \LESS{} variant with averaging.

The general framework of \LESS{}, described above, allows obtaining different variants by adopting different choices in its steps. Those structural choices include the choice of local and global models, splitting the data for local and global learning steps, subset selection techniques, and the way multiple repetitions are exploited. While we numerically explore several variants in Section \ref{sec: Computational Experiments}, we discuss all the structural choices more comprehensively in \ref{sec: Variants}.

\subsection{Linear Models as Local and Global Learners} \label{sec: Linear Models as Local and Global Learners}
Let us explain \LESS{} in a simple scenario where the local and global learners are all linear.
We assume that either the data is centered or, equivalently, every $\vx_{i}$ has its first element equal to one for the model to accommodate an intercept parameter. We consider the standard linear regression model for training local models, that is, $\CL(\vx | \mX_j, \vy_j) = \vv_{j}\tr \vx$ for some $\vv_{j} \in \RR^p$. For instance, assuming $\mX_j$ is full rank\footnote{When $\mX_j$'s are rank deficient, a regularizing parameter $\eta > 0$ can be used to obtain $\vv_{j} = (\mX_j\tr \mX_j + \eta I)\inv \mX_{j}\tr \vy_j$ instead.}, the ordinary least squares solution leads to
$\vv_j = (\mX_j\tr \mX_j)\inv \mX_{j}\tr \vy_j$. Using these parameters, we obtain
\[
z_{ij} = \vv_j\tr \vx_{i} w_{j}(\vx_{i}), \quad i=1, \dots, n; ~ j=1, \dots, m.
\]
For the global learning method $\CG(\vz | \mZ, \vy)$, we again apply the standard linear regression model, but this time regress on the feature vector $\vz$. If we denote the ordinary least squares solution with $\vbeta = [\beta_1, \dots, \beta_m]\tr$, then provided $\mZ$ is full rank, we can write
\begin{equation} \label{eq: linear global beta}
\vbeta = (\mZ\tr \mZ)^{-1} \mZ\tr \vy. 
\end{equation}
Writing the overall prediction at a test point $\vx_0$ as
\begin{equation}
    \hat{y}_0  =  \sum_{j=1}^m (\vv_j\tr \vx_0) [ w_{j}(\vx_0) \beta_j], \label{eq: global learner as linear combination}
  \end{equation}
we see that the overall prediction is a \textit{locally linear combination} of the local predictions, where the coefficient for the $j$'th predictor is $w_{j}(\vx) \beta_j$, which depends both on the proximity of the test point $\vx_{0}$ to the training subset $\mX_j$ and the associated regression coefficient of the $j$'th local model within the global predictor. A related interpretation stems from rewriting the above as
\begin{equation}
\hat{y}_0  = \vx_0\tr \left[\sum_{j=1}^m \vv_j w_{j}(\vx_0) \beta_j \right], \label{eq: local interpretation}
\end{equation}
which implies that the prediction for a test sample is \textit{also} a locally linear combination of $\vx_0$. Here, the coefficients of the linear combination are locally determined, that is, they depend on $\vx_0$. It is important to note that whenever linear models (\textit{e.g.}, ridge regression, Lasso, elastic net) are used for local and global learners, we obtain a sample-dependent weight for each feature as in \eqref{eq: local interpretation}. This immediately lends itself to local explanations for the test samples. Consequently, after training \LESS{} with linear models, we obtain feature attributions for local explanation without the need for post-hoc methods like the infamous SHAP and its extensions \citep{NIPS2017_7062}.

{\noindent}\textbf{Numerical demonstration.} 
Figure \ref{fig:1dex} illustrates how \LESS{} works, as described in this example, on a dataset generated by adding noise to a trigonometric function. Blue circles represent the samples used for training and orange circles show the samples in the test set. The black line segments in the figure are the local predictions obtained by fitting the linear regressors for each sample subset $\mX_j$. These predictions can be seen to follow the general pattern of the dataset. For the global method, we have again applied the standard linear regression model to the transformed data.

\begin{figure}[ht]
\centerline{
\begin{minipage}{0.6\textwidth}
  \includegraphics[scale = 0.35]{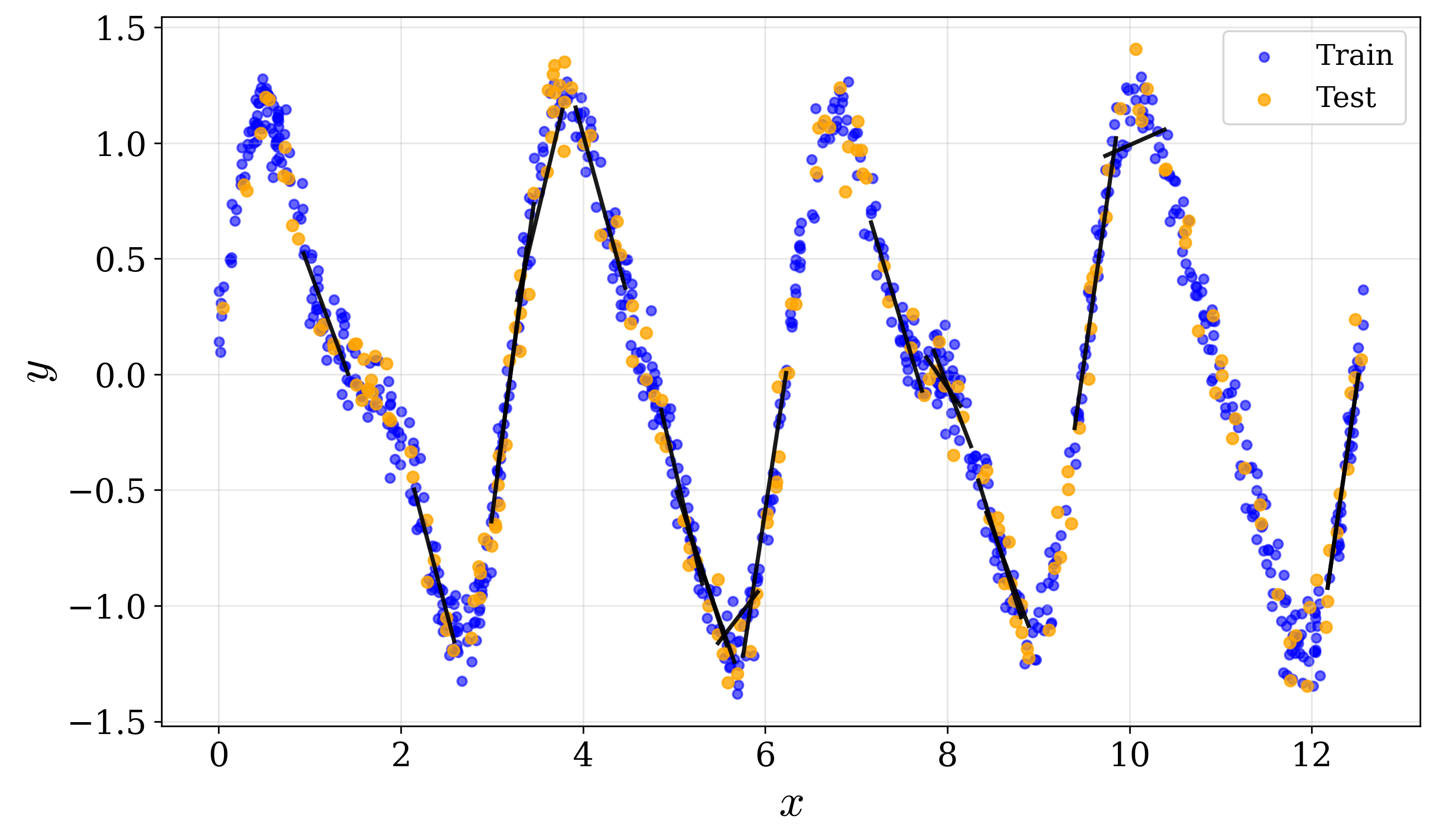}
  \end{minipage}
  \hfill
  \begin{minipage}{0.35\textwidth}
  \caption{Demonstration of local models in \LESS{} on one-dimensional synthetic dataset.}
  \label{fig:1dex}
  \end{minipage}
  }
\end{figure}

Figure \ref{fig:mvsr} illustrates how the number of subsets ($m$) and the number of replications ($b$) for averaging change the behavior of \LESS{}. The original unknown model is the same trigonometric function used in Figure~\ref{fig:1dex}. We first train each \LESS{} model using $n=200$ randomly generated samples with the specified hyperparameter values for $m$, $b$, and $k=n/m$. The subsets are constructed by randomly sampling $m$ anchor points as we discussed in Section \ref{sec: General Framework}. The trained models are used to plot the corresponding subfigures in Figure \ref{fig:mvsr}. In each row of the figure, we notice that increasing the number of subsets assists \LESS{} in identifying the nonlinear structure of the unknown model. However, the last column also shows that when the number of subsets is too high, the variance also increases. This can be considered overfitting. We observe the smoothing effect of averaging as the number of replications increases at each column of the figure. For the example problem, the unknown function is closely approximated when the number of subsets and the number of replications are both 20.

\begin{figure}[h!]
\centerline{
\begin{minipage}{0.25\textwidth}
\begin{tikzpicture}
  \node (img11)  {\includegraphics[scale=0.15]{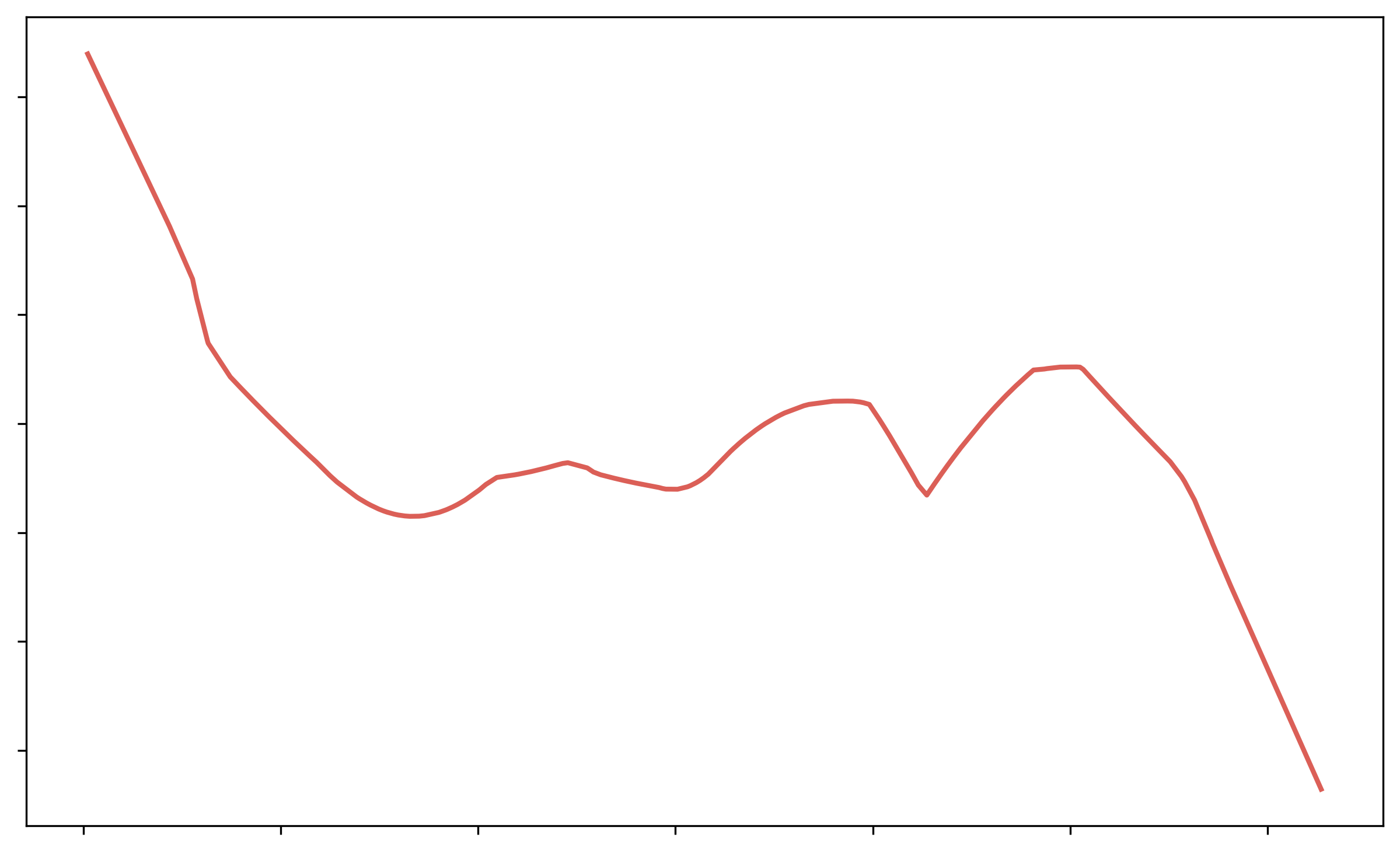}};
  \node[below=of img11, yshift=1cm] (img21)  
{\includegraphics[scale=0.15]{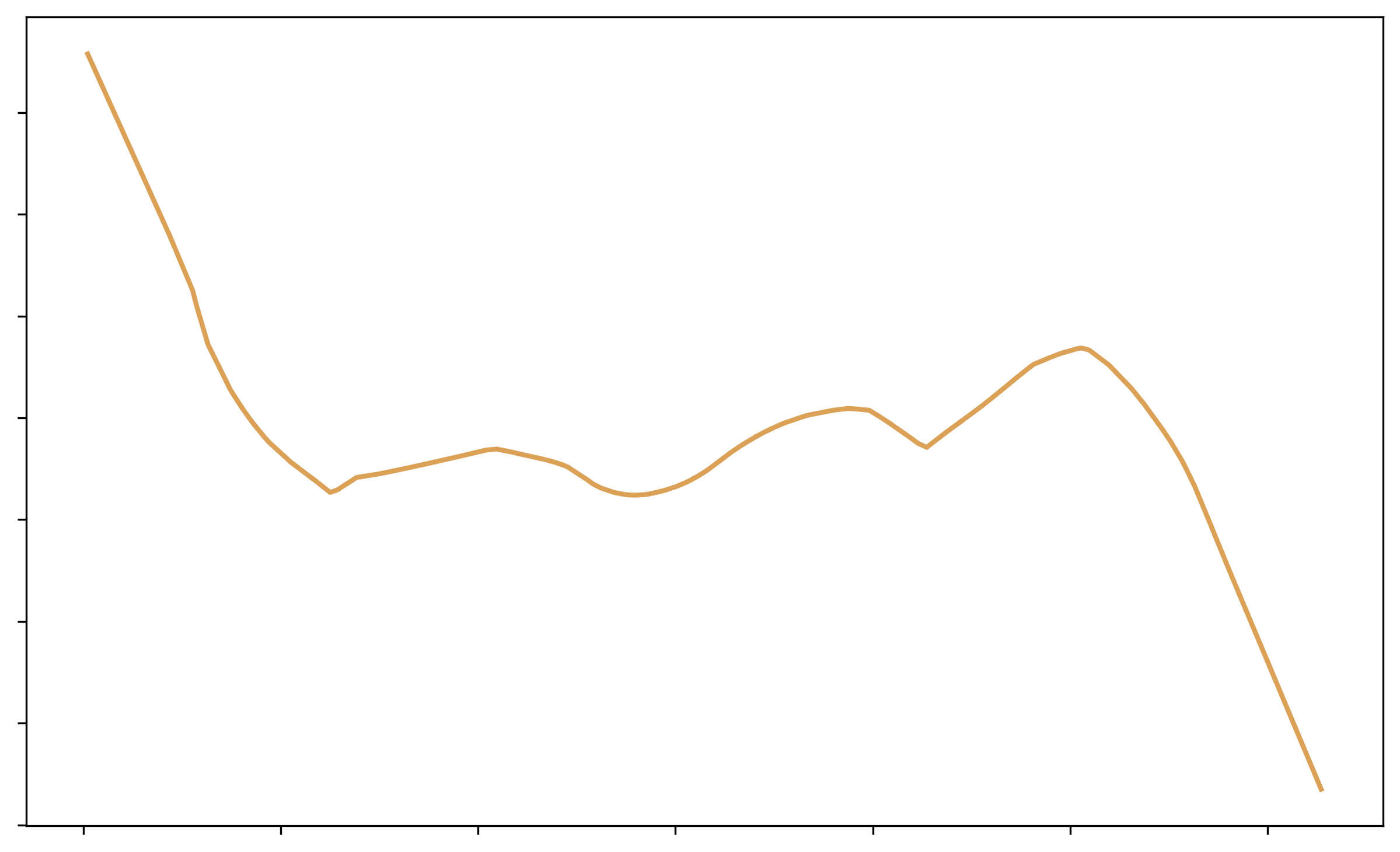}};
  \node[below=of img21, yshift=1cm] (img31)  {\includegraphics[scale=0.15]{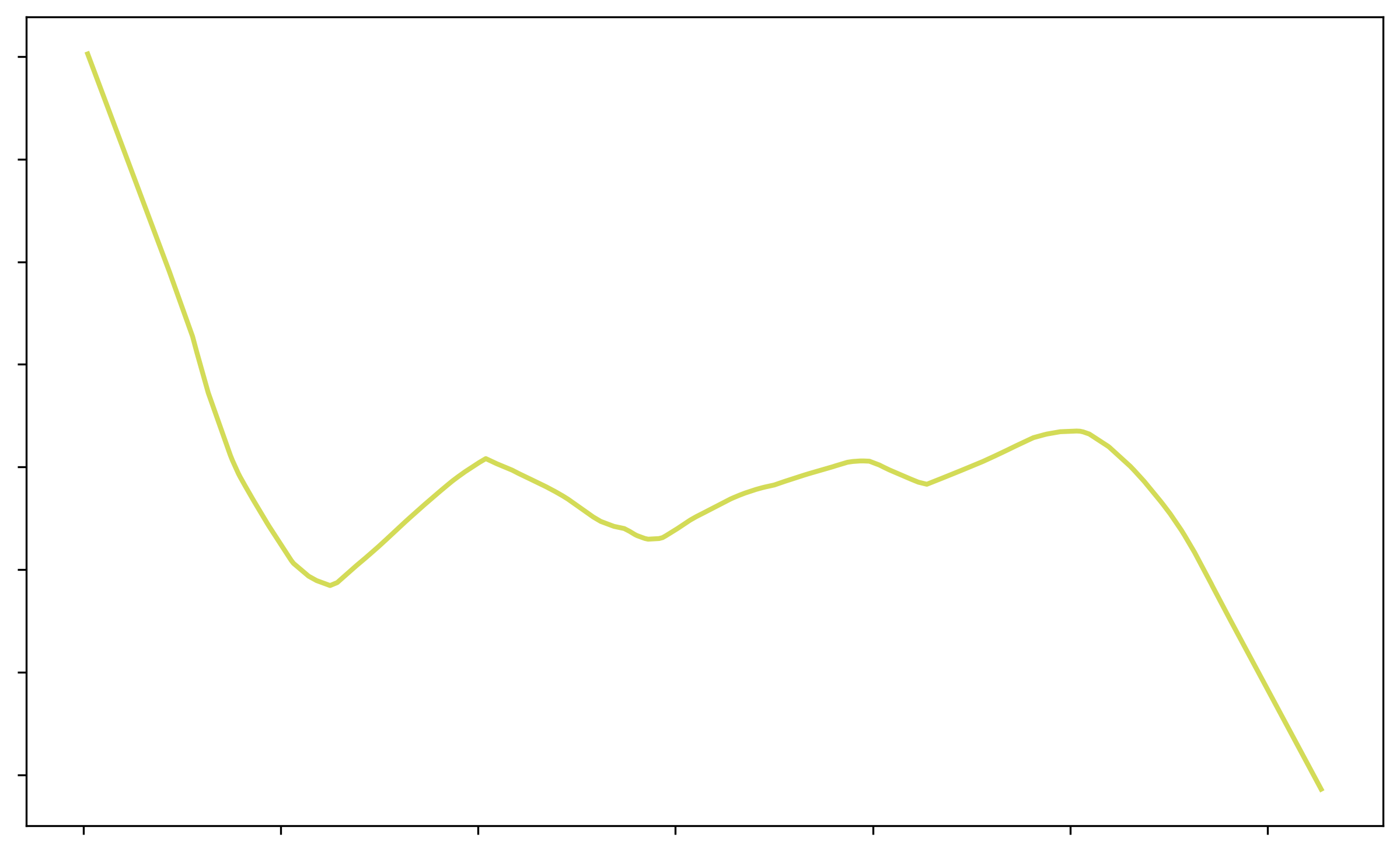}};
  \node[above=of img11, node distance=0cm, rotate=0, anchor=center,yshift=-1cm,font=\color{black}] {\scriptsize{$m = 5, k = 40, r = 5$}};
  \node[above=of img21, node distance=0cm, rotate=0, anchor=center,yshift=-1cm,font=\color{black}] {\scriptsize{$m = 5, k = 40, r = 10$}};
\node[above=of img31, node distance=0cm, rotate=0, anchor=center,yshift=-1cm,font=\color{black}] {\scriptsize{$m = 5, k = 40, r = 20$}};
 \end{tikzpicture}
\end{minipage}%
\hspace{0.2 cm}
\begin{minipage}{0.25\textwidth}
\begin{tikzpicture}
  \node (img12)  {\includegraphics[scale=0.15]{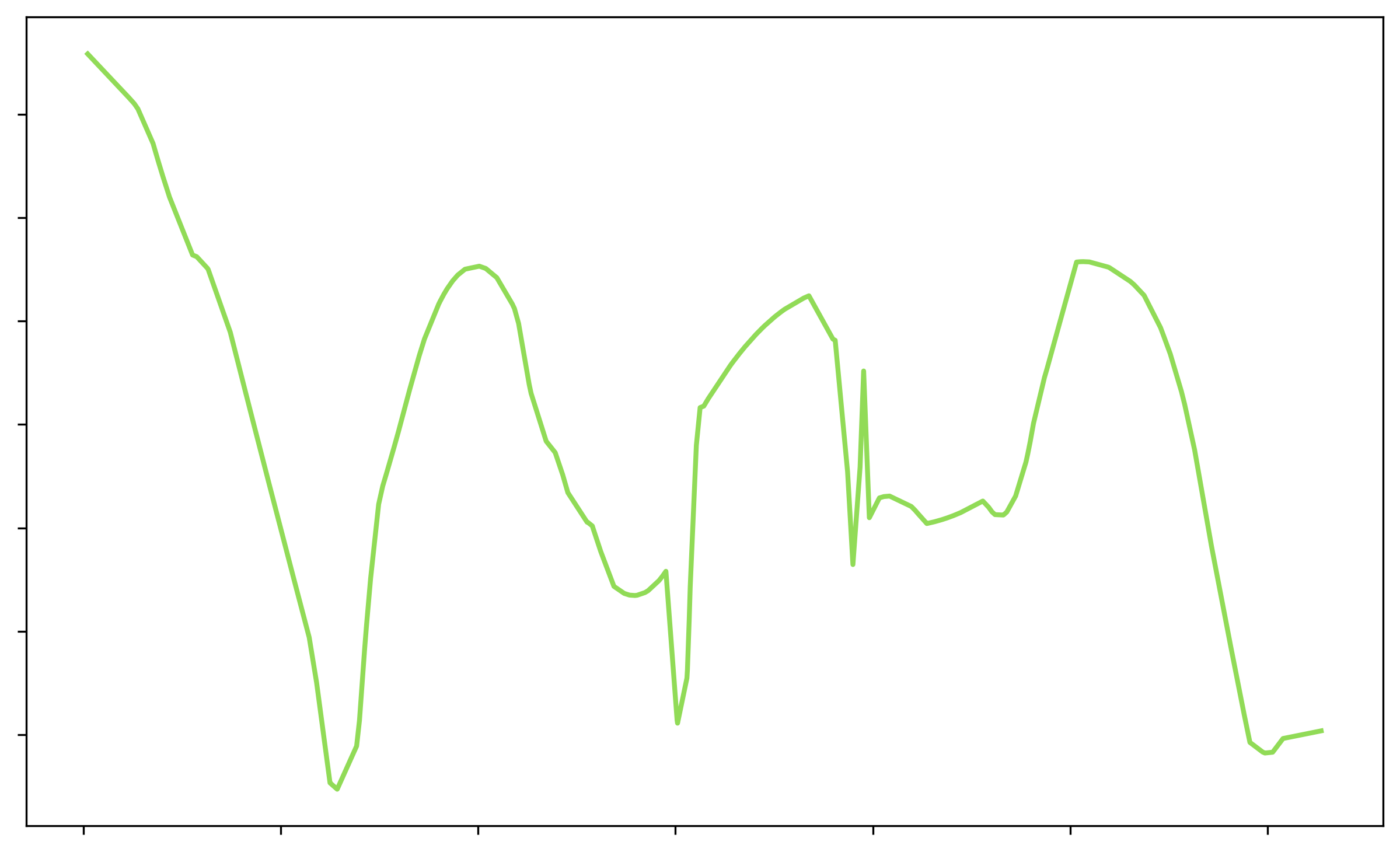}};
  \node[below=of img12, yshift=1cm] (img22)  {\includegraphics[scale=0.15]{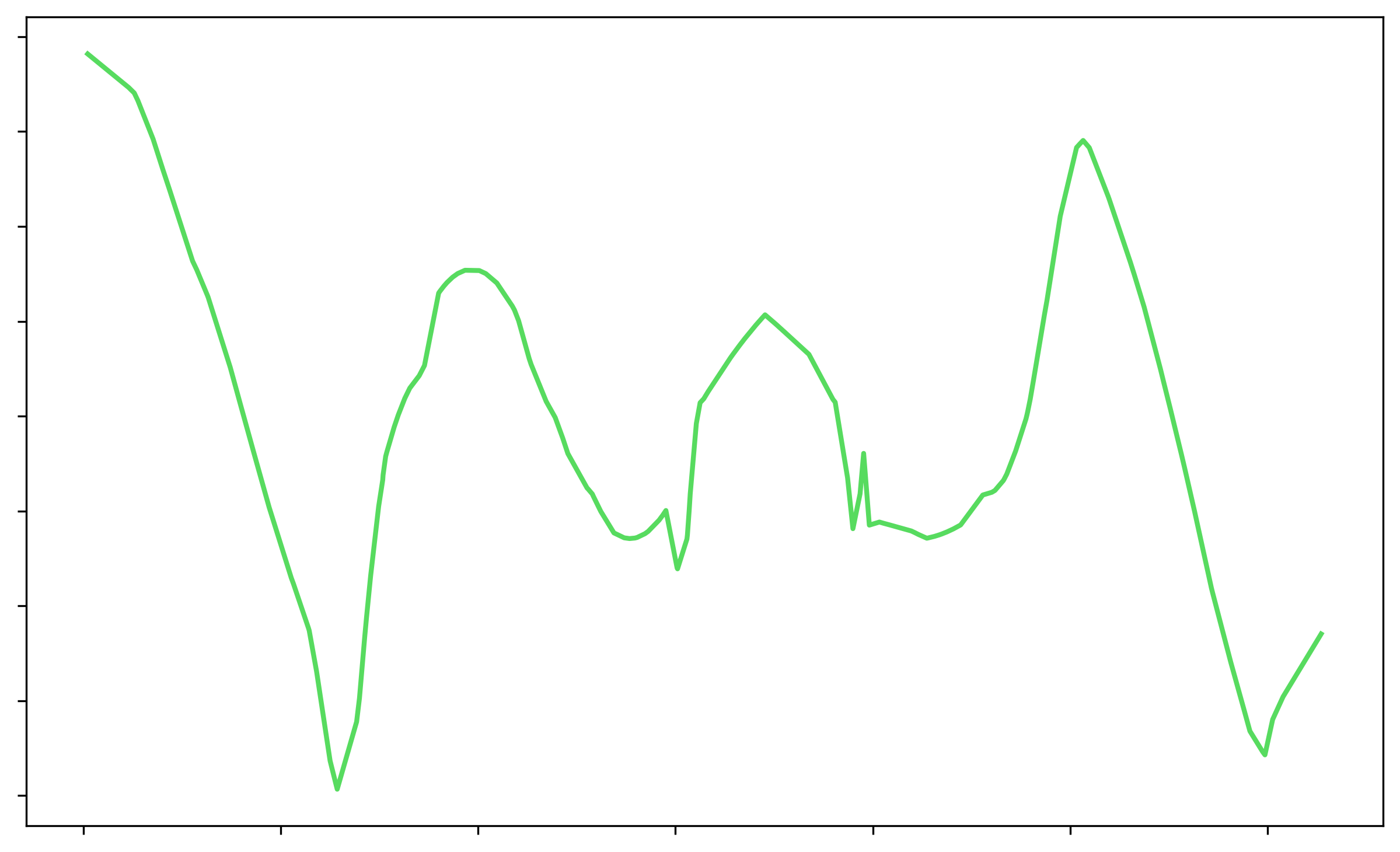}};
  \node[below=of img22, yshift=1cm] (img32)  {\includegraphics[scale=0.15]{figures/v3/trigo/LESSAVRegressor_sub10_min20_est10.png}};
  \node[above=of img12, node distance=0cm, rotate=0, anchor=center,yshift=-1cm,font=\color{black}] {\scriptsize{$m = 10, k = 20, r = 5$}};
  \node[above=of img22, node distance=0cm, rotate=0, anchor=center,yshift=-1cm,font=\color{black}] {\scriptsize{$m = 10, k = 20, r = 10$}};
\node[above=of img32, node distance=0cm, rotate=0, anchor=center,yshift=-1cm,font=\color{black}] {\scriptsize{$m = 10, k = 20, r = 20$}}; \end{tikzpicture}
\end{minipage}%
\begin{minipage}{0.25\textwidth}
\begin{tikzpicture}
  \node (img13)  {\includegraphics[scale=0.15]{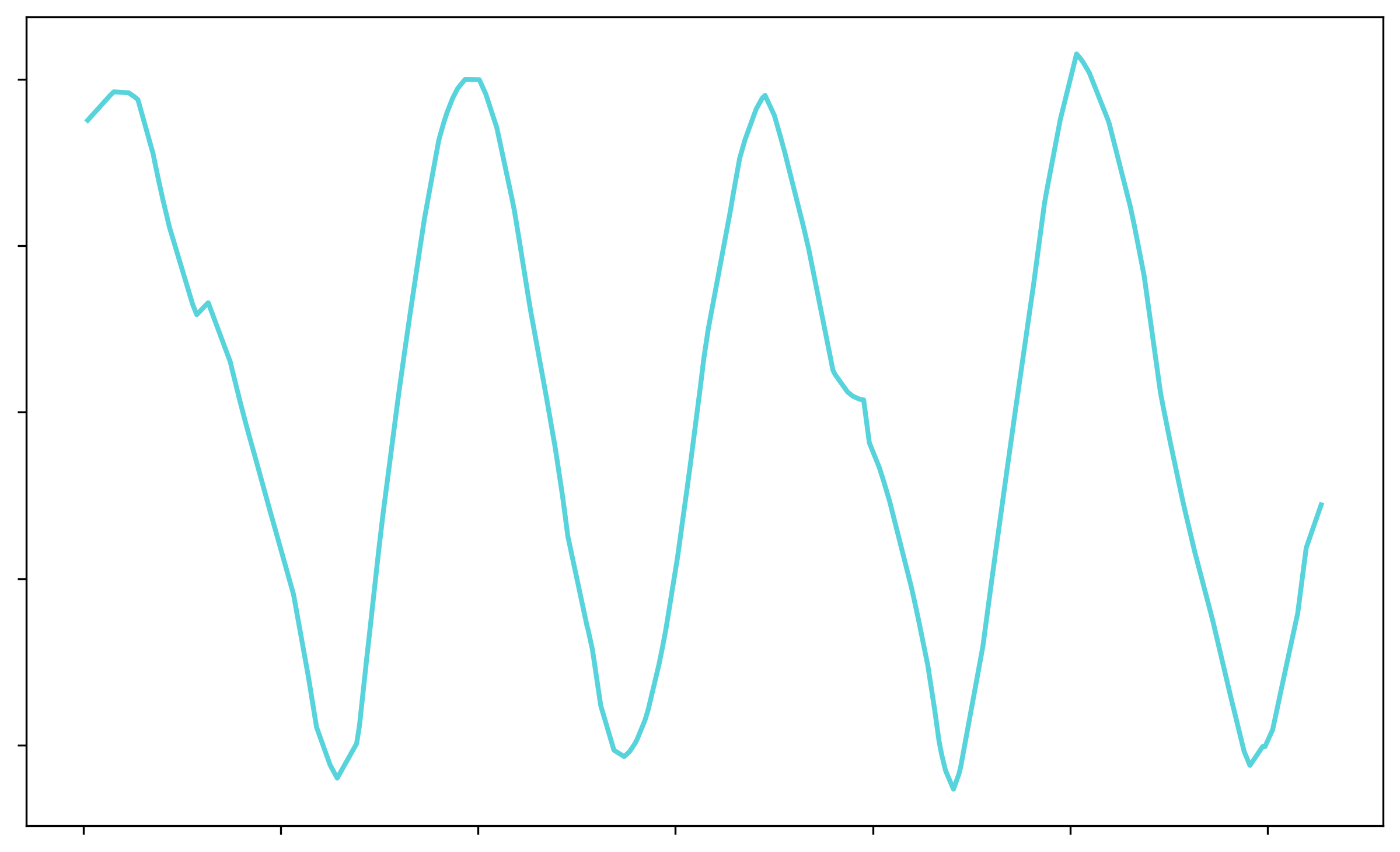}};
  \node[below=of img13, yshift=1cm] (img23)  {\includegraphics[scale=0.15]{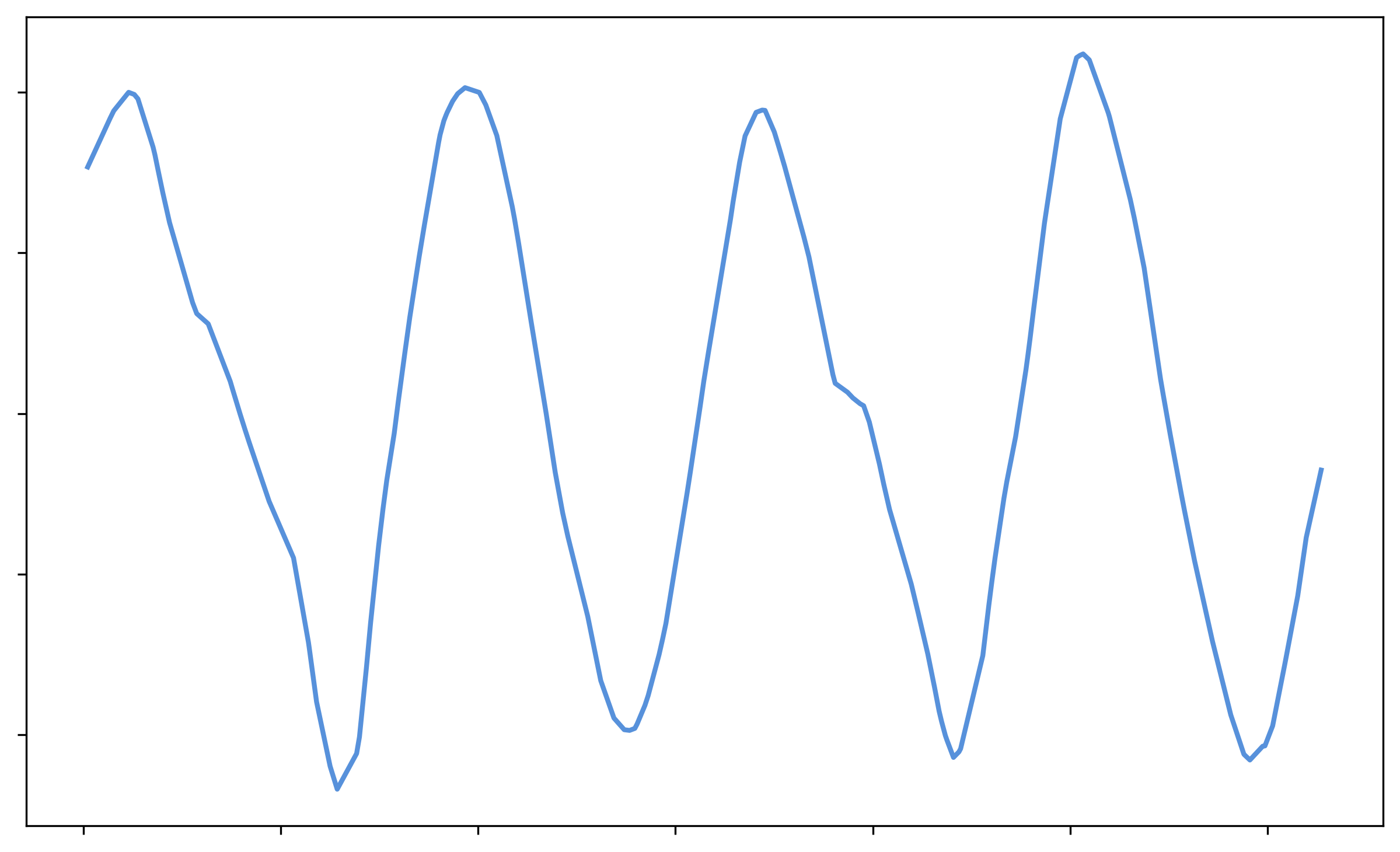}};
  \node[below=of img23, yshift=1cm] (img33)  {\includegraphics[scale=0.15]{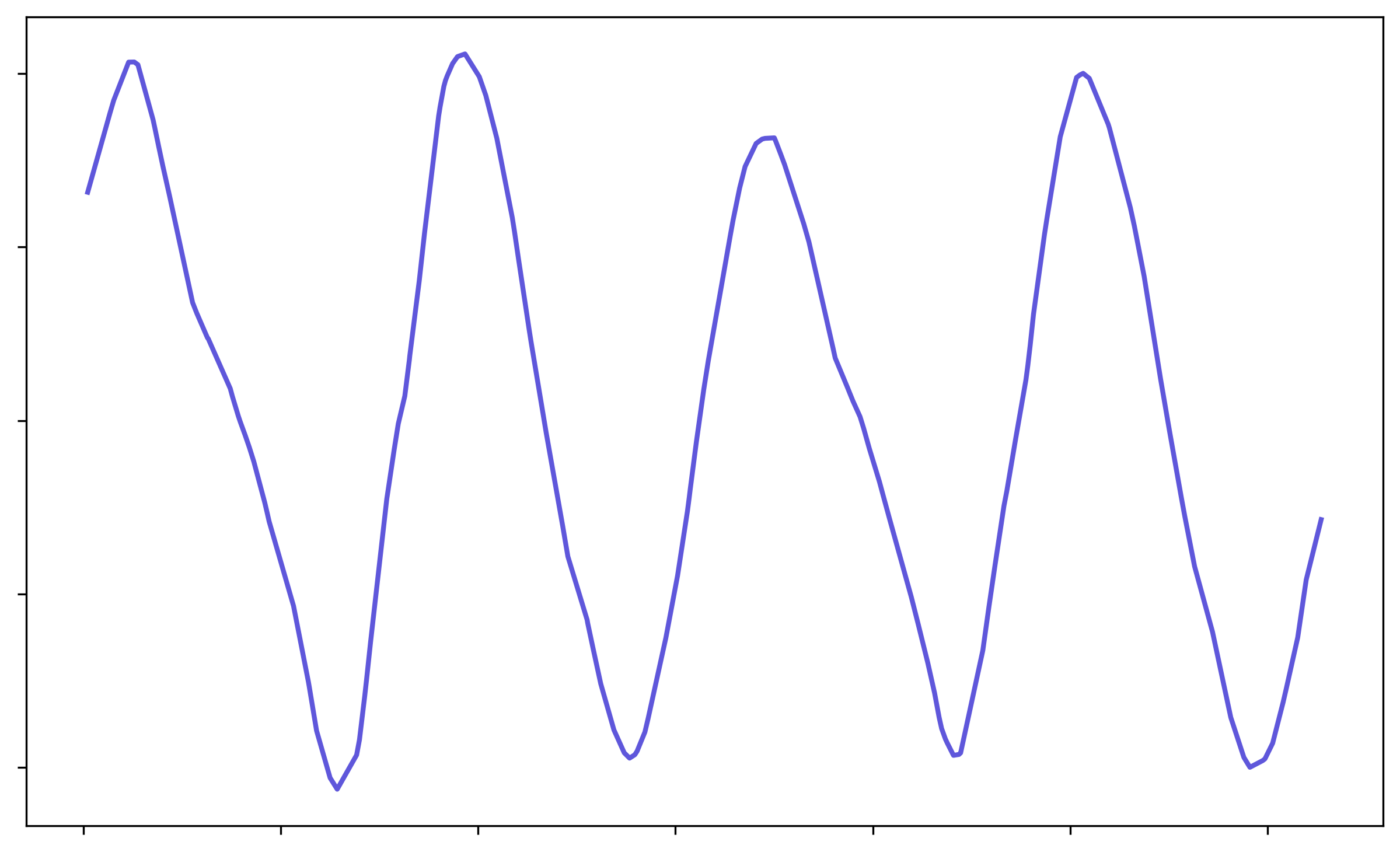}};
  \node[above=of img12, node distance=0cm, rotate=0, anchor=center,yshift=-1cm,font=\color{black}] {\scriptsize{$m = 20, k = 10, r = 5$}};
  \node[above=of img22, node distance=0cm, rotate=0, anchor=center,yshift=-1cm,font=\color{black}] {\scriptsize{$m = 20, k = 10, r = 10$}};
\node[above=of img32, node distance=0cm, rotate=0, anchor=center,yshift=-1cm,font=\color{black}] {\scriptsize{$m = 20, k = 10, r = 20$}}; 
 \end{tikzpicture}
\end{minipage}%
\begin{minipage}{0.25\textwidth}
\begin{tikzpicture}
  \node (img14)  {\includegraphics[scale=0.15]{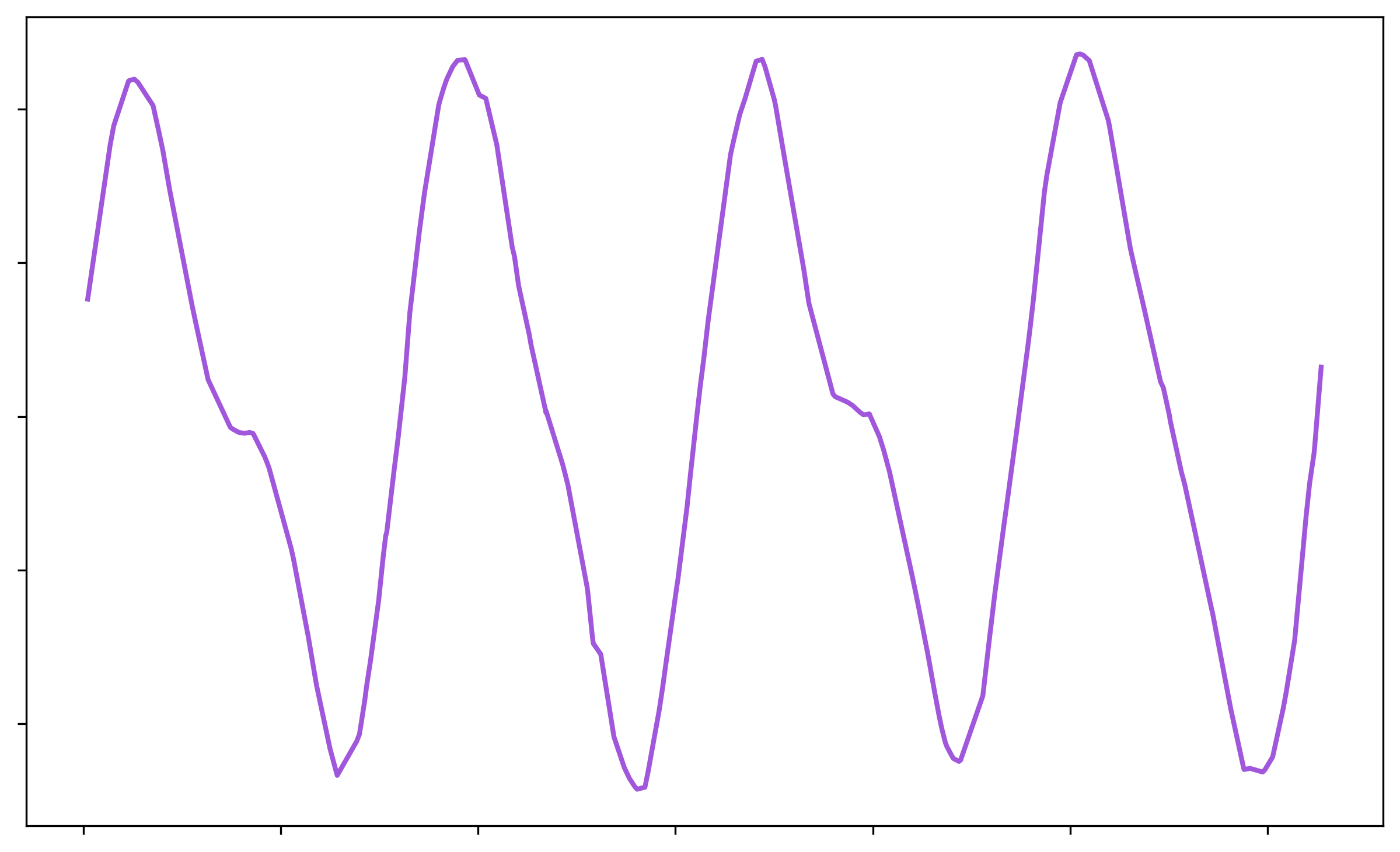}};
  \node[below=of img14, yshift=1cm] (img24)  {\includegraphics[scale=0.15]{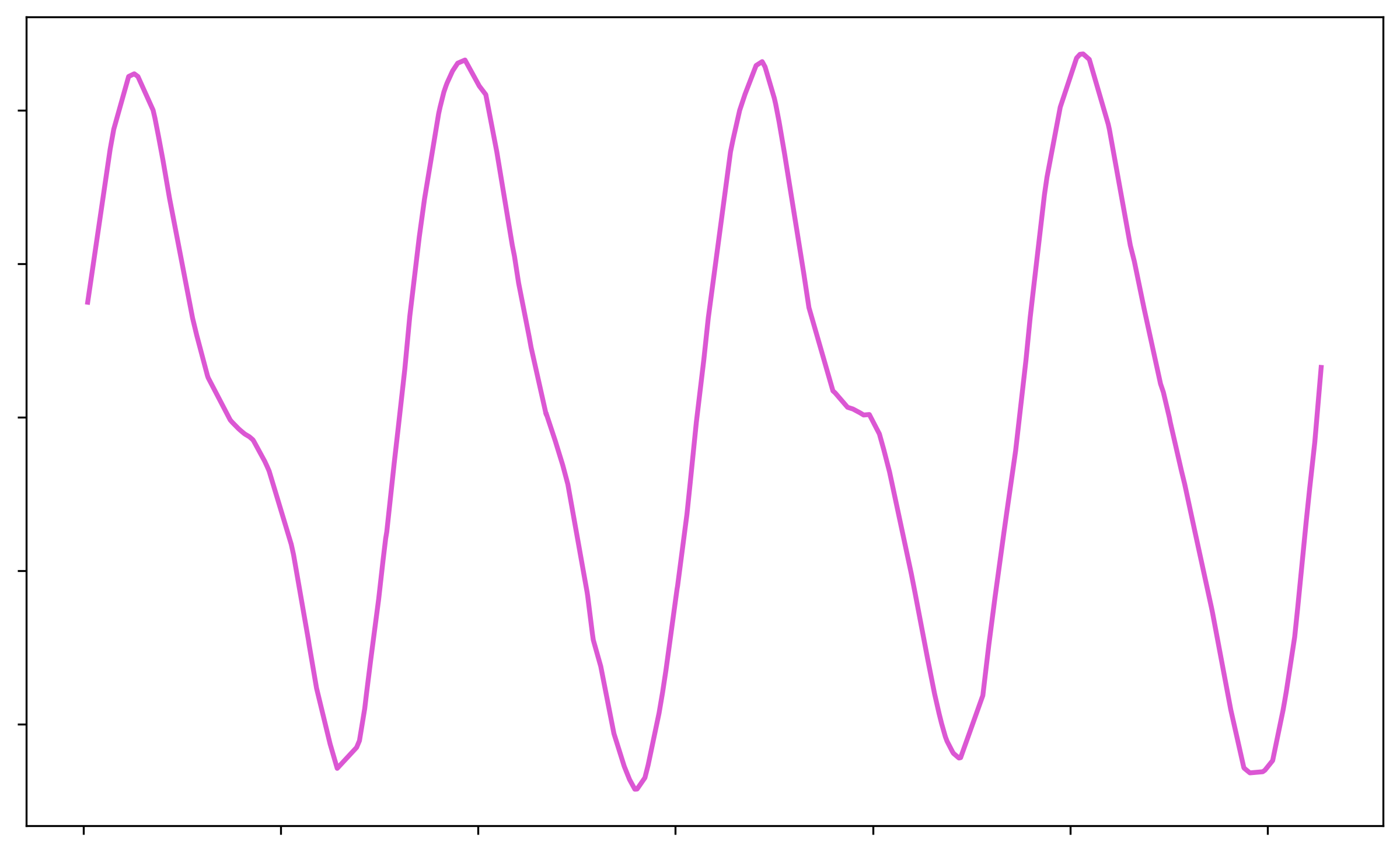}};
  \node[below=of img24, yshift=1cm] (img34)  {\includegraphics[scale=0.15]{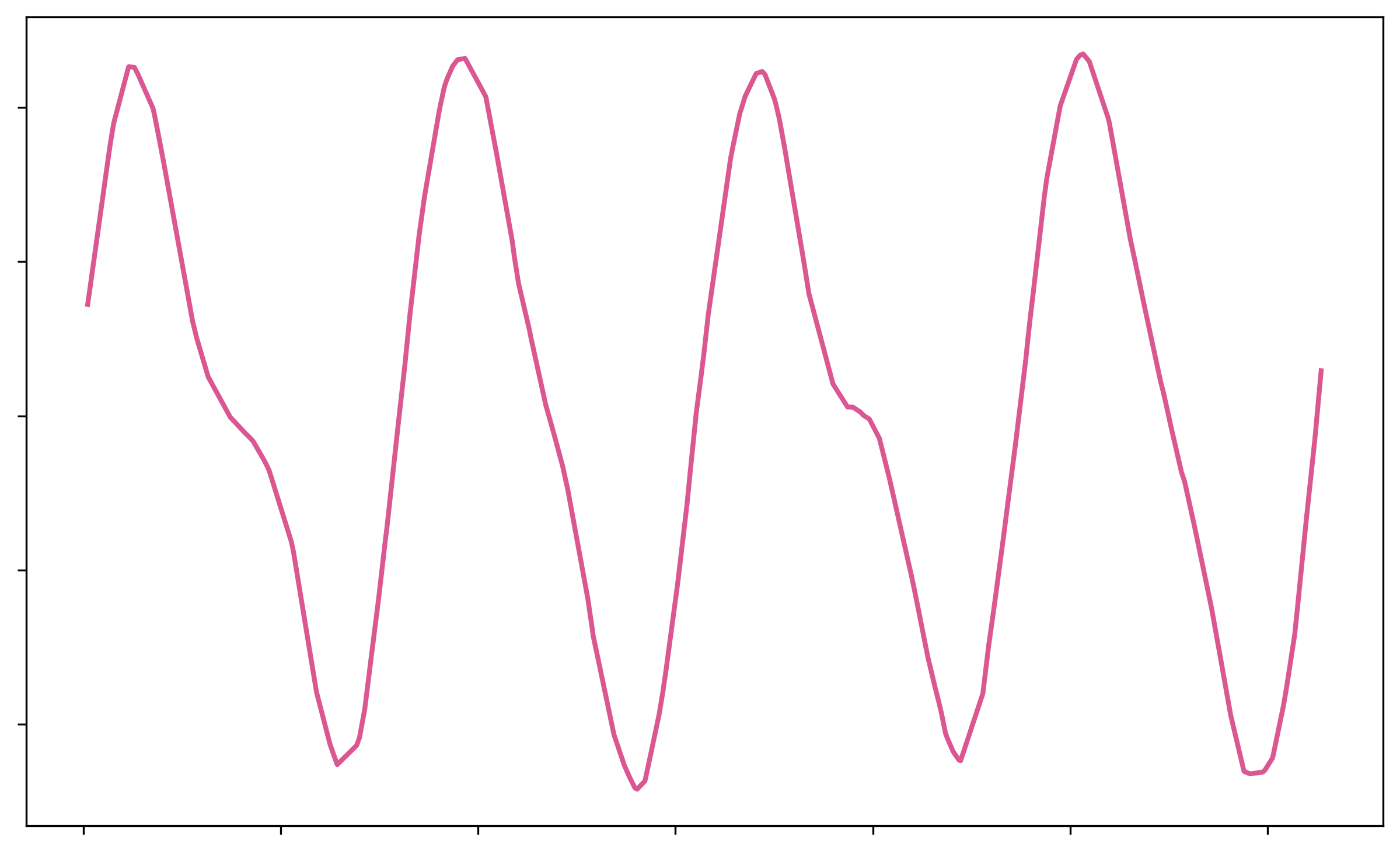}};
  \node[above=of img12, node distance=0cm, rotate=0, anchor=center,yshift=-1cm,font=\color{black}] {\scriptsize{$m = 50, k = 4, r = 5$}};
  \node[above=of img22, node distance=0cm, rotate=0, anchor=center,yshift=-1cm,font=\color{black}] {\scriptsize{$m = 50, k = 4, r = 10$}};
\node[above=of img32, node distance=0cm, rotate=0, anchor=center,yshift=-1cm,font=\color{black}] {\scriptsize{$m = 50, k = 4, r = 20$}};
 \end{tikzpicture}
\end{minipage}%
}
  \caption{The effects of the number of replications ($b$), the number of subsets ($m$), and the number of neighbors ($k$). Notice that $n = m \times k$. The subsets are constructed by randomly selecting $m$ anchor points.}
  \label{fig:mvsr}
\end{figure}


\subsection{Effect of Weighting} \label{sec: Effect of Weighting}
Recall that the feature vector $\vz_{i}$ for the input point $\vx_{i}$ is constructed such that its $j$'th element is $z_{ij} = w_{j}(\vx_{i}) \CL(\vx_i |  \mX_j, \vy_j)$, where $w_{j}(\vx_{i})$ is a distance-based weight for subset $j$ at the input point $\vx_{i}$. In this section, we investigate the effect of the weighting function on the behavior of \LESS{}. For analysis's sake, assume that linear regression is used for both the local and the global learning phases. Recall that this leads to
\[
\CL(\vx_i |  \mX_j, \vy_j) = \vx_{i}\tr \vv_{j} ~ \text{ with } ~ \vv_{j} = (\mX_{j}\tr \mX_{j})^{-1} \mX_{j}\tr \vy_{j}.
\]
for the local models and to the prediction
\[
\hat{y}_0  =  \sum_{j=1}^m (\vv_j\tr \vx_0) [ w_{j}(\vx_0) \beta_j] ~ \text{ with } ~ \vbeta = (\mZ\tr \mZ)^{-1} \mZ\tr \vy
\]
for the global model. We specifically consider the weighting function in \eqref{eq: distance-based weights}, although it will become clear that the subsequent discussion can be applied to many other sensible choices for $w_{j}(\vx)$. In \eqref{eq: distance-based weights} we have $w_{j}(\vx) \propto \exp\{-\lambda d(\vx, \bar{\vx}_{j}) \}$ with $\sum_{j = 1}^{m} w_{j}(\vx) = 1$ for all $\vx$. Therefore, the hyperparameter $\lambda \geq 0$ controls the amount of impact of the subsets on each other. 

We discuss the two extreme cases for the hyperparameter $\lambda = 0$ and $\lambda \rightarrow \infty$. The first extreme case is $\lambda = 0$, which removes the effect of the distance function.
\begin{proposition}\label{proposition::OLS}
Let $\mV$ be the $p \times m$ matrix whose $j$'th column is $\vv_{j}$. If $\mV$ has rank $p$, then \LESS{} reduces to ordinary least squares (OLS) with $\lambda = 0$.
\end{proposition}
{\noindent}The proof is given in \ref{appendix::proofs}. It is worth emphasizing that it is \textit{not} necessary for the subsets $(\mX_{j}, \vy_{j})$, $j = 1, \ldots, m$ to be disjoint, nor is it for them to cover all data points for the above analysis to hold. The only requirement for the exact recovery of OLS is that the rank of $\mV$ is $p$.

The second extreme case is obtained with $\lambda \rightarrow \infty$, which corresponds to choosing the prediction of the local model whose center is at the minimum distance from the testing point. Our result for this case requires a specific way to choose the subsets, which is typically met by clustering methods.

\begin{proposition}\label{proposition::lambda-infty}
Assume each data point $\vx_{i}$ belongs to the subset to whose center it is closest. Then, as $\lambda \rightarrow \infty$, \LESS{} reduces to a collection of non-interacting local models: If a test point $\vx_{0}$ is closest to subset $j$, the prediction by \LESS{} at $\vx_{0}$ is given by $\vv_{j}\tr \vx_{0}$. 
\end{proposition}
{\noindent}The proof is given in \ref{appendix::proofs}. The two extreme cases demonstrate the critical role of the distance function in the communication level among the subsets. Therefore, it is natural to take $\lambda$ as a hyperparameter that could be tuned, \textit{e.g.}, by cross-validation.

\begin{remark}[Unnormalized weights]
The weights in \eqref{eq: distance-based weights} are normalized so that they sum to one. It is also possible to use unnormalized weights when generating the feature vectors, such as $w_{j}(\vx) = \exp\{ -\lambda d(\vx, \bar{\vx}_{j}) \}$ following the example of \eqref{eq: distance-based weights}. Such a choice may be less intuitive; for example, the above analysis on the effect of $\lambda$ does not hold with unnormalized weights. However, unnormalized weights may be useful when it is desired to have predictions with small magnitudes as the test points are far away from all of the subsets, to reduce the prediction variance.
\end{remark}

\section{Computational Experiments} \label{sec: Computational Experiments}
We have tested \LESS{} on publicly available datasets with varying sizes ($n \times p$): \texttt{abalone} ($4,176 \times 7$), \texttt{airfoil} ($1,503 \times 5$), \texttt{housing} ($505 \times 13$), \texttt{cadata} ($20,640 \times 8$), \texttt{ccpp} ($9,568 \times 4$), \texttt{energy} ($19,735 \times 26$), and \texttt{cpu} ($8,192 \times 12$). 
These datasets are all from the UCI repository by \cite{Dua2017} except \texttt{cadata} \citep{pace1997sparse}. 

In all the experiments, we have normalized the input as well as the output. Our open source implementation on Python is available online\footnote{\url{https://github.com/sibirbil/LESS}}, which can be used to reproduced all the results in this section.

\begin{remark}[Default Choices for \LESS{}] \label{rem: default choices for LESS}
The following set of choices for \LESS{} was mostly used during the experiments. The percentage of samples is set to determine the number of neighbors ($k$), and then the number of subsets $m=\lceil n/k \rceil$. The subsets are created by selecting $m$ anchors uniformly. For weighting in \eqref{eq: distance-based weights}, we use for any $\vx, \vx' \in \mathbb{R}^{p}$ the distance function $d(\vx, \vx') = \|\vx- \vx'\|_{2}$ with $\lambda = m^{-2}$. The number of replications ($b$) is set to 20.
\end{remark}

In the subsequent part, we distinguish the averaging and boosting variants of \LESS{} as \LESSA{} and \LESSB{}, respectively. In all experiments that require performance measurement and hyperparameter tuning, we have applied $5 \times 4$ nested cross-validation; that is, five-fold cross-validation is used for train-test split, and four-fold cross-validation is used for hyperparameter tuning. Mean-squared errors (MSE) are reported after averaging over five test splits.

\subsection{Ablation Study} \label{sec: Ablation Study}

First, an ablation study was conducted that shows the crucial roles of both the weighting and global learning steps in the success of \LESS{}. To this end, we considered \LESSA{} and \LESSB{} variants, each using the default choices in Remark \ref{rem: default choices for LESS} and linear regression as their local and global models. 


For \LESSA{} and \LESSB{} separately, we stripped off either the weighting part (NoW-G), the global learning part (W-NoG), or both (NoW-NoG) to obtain different `ablated' methods. More concretely, NoW-G corresponds to constructing the feature vector $\vz_i$ at $\vx_i$ as $z_{ij} = \mathcal{L}(\vx_i | \mX_{j}, \vy_{j}), j=1, \dots, m$, W-NoG corresponds to predicting the output for a test point $\vx_0$ with $\hat{y}_0 = \sum_{j = 1}^{m} \mathcal{L}(\vx_0 | \mX_{j}, \vy_{j}) w_{j}(\vx)$, and NoW-NoG corresponds to predicting the output for a test point $\vx_0$ with $\hat{y}_0 = \frac{1}{m}\sum_{j = 1}^{m} \mathcal{L}(\vx_0 | \mX_{j}, \vy_{j})$. For each method, we have only tuned the number of subsets $m$ from the set $\{5, 10, 20, 50\}$. 

Figure \ref{fig:abst} compares \LESSA{} and \LESSB{} against their respective ablated versions obtained as described above. These results demonstrate that the performance of the bare-bones method (NoW-NoG) improves in all problems with the weighting and global learning steps, showing the importance of those steps. Consequently, \LESS{} outperforms the other methods.

\begin{figure}[ht]
\begin{minipage}{0.70\textwidth}
\centerline{
\includegraphics[width = 1\linewidth]{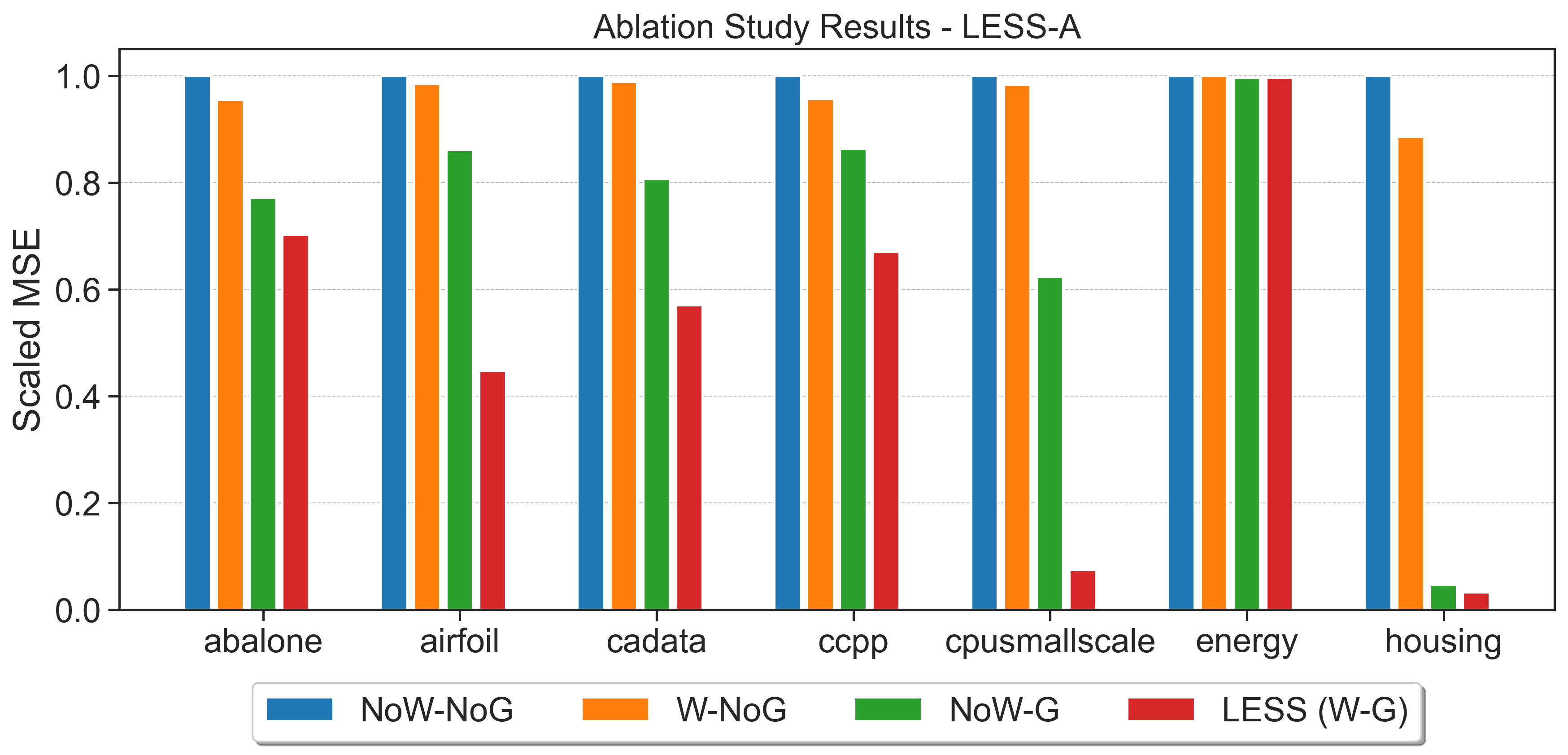}}
\centerline{
\includegraphics[width = 1\linewidth]{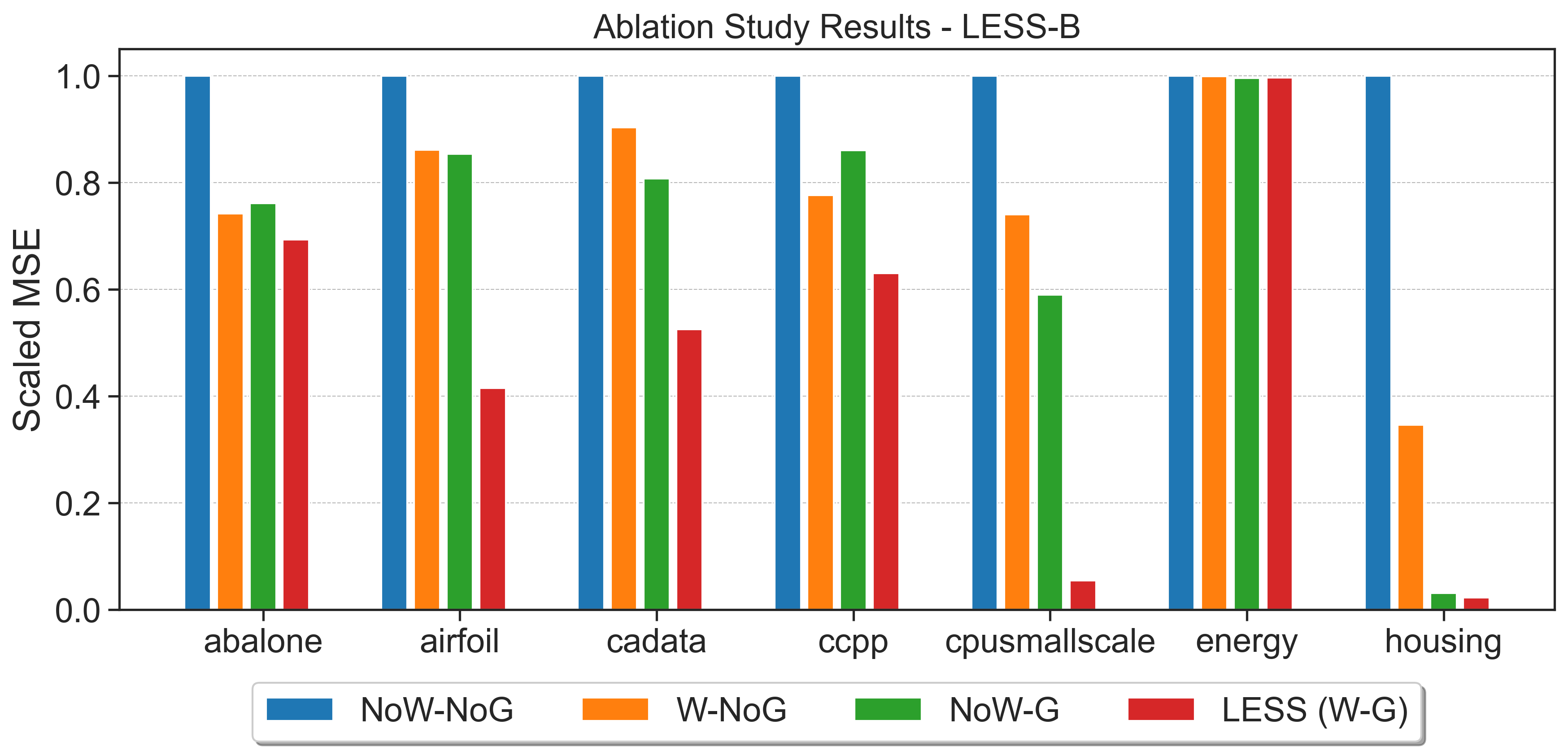}
}    
\end{minipage}
\hfill
\begin{minipage}{0.25\textwidth}
\caption{Results of the ablation study. The vertical axis is obtained by scaling the MSE values with the maximum one among  NoW-NoG (no weighting, no global learning), W-NoG (weighting, no global learning),  NoW-G (no weighting, global learning),  \LESS{} (default).}    
\label{fig:abst}
\end{minipage}


\end{figure}

\subsection{Comparison Against Other Methods} \label{sec: Comparison Against Other Methods}

To compare the performance of \LESS{} against other regression methods, we have selected 10 well-known methods. 
These methods are Random Forest (RF), $k$-Nearest Neighbour (KNN), multi-layer perceptron (MLP), Local Linear Regression (LocR), Magging (MAG), Gaussian process regression (GPR), Light Gradient Boosting Machine (LGBM), Extreme Gradient Boosting (XGB), and two fuzzy learning algorithms: the Adaptive Network-based Fuzzy Inference System (ANFIS) and the Takagi-Sugeno-Kang (TSK) algorithm.  
All these methods, except Local Linear Regression (LocR) and Magging (MAG), are available through \texttt{scikit-learn} package \citep{sklearn}. Like \LESS{}, we have also implemented LocR and MAG using main \texttt{scikit-learn} classes, and our implementations are based on Section 6.3 of \cite{hastie_09} and the pseudocode\footnote{We have realized that occasionally the quadratic programming solver fails to find a solution. In those cases, our implementation assigns equal weights to the estimators.} in \cite{BuhlmannMeinshausen2015Magging}, respectively. 

Although \LESS{} is a generic approach that allows various choices for its hyperparameters, we have decided to tune only a few of those, so that we can demonstrate ease-of-use and good performance of \LESS{}. In addition to the default choices in Remark \ref{rem: default choices for LESS}, we have also tried the following alternatives in cross-validation: decision tree regressors and linear regression for local learning, and $\lambda \in \{0.01, 0.1\}$ for the kernel parameter. Random Forest was selected for global learning. The number of subsets was tuned over $\{10, 20\}$. Finally, each compared method is tuned using cross-validation with the hyperparameter sets given in Table \ref{tab:hyps} in \ref{appendix::comp-against-others}.


The results in Table \ref{tab:results} show that \LESS{} ranks the best in three out of seven problems among all methods in terms of mean squared error values. LGBM outperforms the other methods in three problems. This performance is followed by MLP, which obtains the least average error for one problem.

We complement Table \ref{tab:results} with statistical many-to-one comparisons between LESS and the other methods. Table \ref{tbl: confidence intervals} shows the two-way ANOVA results using observations $\text{MSE}(i, j, k)$, the MSE for dataset $i$, method $j$, and repetition $k$, based on the assumption $\text{MSE}(i, j, k) = \mu_{i} + \nu_{j} + e_{i, j, k}$ with $e_{i, j, k} \overset{\text{i.i.d.}}{\sim} \mathcal{N}(0, \sigma^{2})$ for $i = 1, \ldots, 7; j = 1, \ldots, 12; k = 1, \ldots, 5$. Letting \LESSB{} be method $1$, we provide the 95\% simultaneous confidence intervals for the differences $\nu_{j} - \nu_{1}$ in Table \ref{tbl: confidence intervals}. The $p$-values below the confidence intervals correspond to the hypotheses that \LESSB{} and the compared method have the same performance in terms of (expected) MSE. The reported p-values and the confidence intervals suggest a significant superiority of \LESSB{} over the compared methods, except for LGBM.

The methods were compared with respect to their rankings in terms of average MSEs. Let $R_{i, j}$ be the ranking of the $j$'th method when it is applied to the $i$'th dataset. The averages $\frac{1}{7} \sum_{j = 1}^{7} R_{i, j}$ for each method $j$ are reported in the last line of Table \ref{tab:results}. The dataset-specific rankings, along with the average rankings over the datasets, are given in the \ref{appendix::comp-against-others} in Figure \ref{fig: Algorithm Rankings}. It can be observed that \texttt{LESS-B} and LGBM tie as they have the lowest overall ranking, while \texttt{LESS-B} outperformed LGBM in four of the seven datasets.

We have also conducted Friedman's non-parametric test to compare all the methods to \LESS{}, as suggested in \citet{Demsar_2006}. Friedman's test is based on the rankings of the methods in terms of the average MSEs given in Table \ref{tab:results}. Along with those averages, we also report the $p$-values on the last line of Table \ref{tbl: confidence intervals}, where each $p$-value corresponds to the null hypotheses that LESS and the compared method have the same (expected) ranking. Friedman’s test is known to be generally less powerful than the two-way ANOVA. This is also the case for our experiments, see the p-values in Table \ref{tbl: confidence intervals}. Still, the p-values are arguably sufficiently small except for LGBM and XGB, hence the test results largely support the findings of the two-way ANOVA related to the significance of the superiority of LESS. 

\subsection{Computation Times} \label{sec: Computation Times}

Figure \ref{fig:timings} shows the computation times in seconds obtained with \LESS{} for two datasets on a personal computer running macOS 15.7.1 with Apple M1 Pro Chip equipped with 32GB memory and 10 cores. The scatter plots depict that the computation times increase as the number of subsets increases, as expected. The bar plots show the computation times (in logarithmic scale) of other methods with the same datasets. Except for the MLP, we have used the default values suggested by the package maintainers. In the MLP we set the maximum number of iterations to 5,000 to ensure convergence. The dashed horizontal line is the computation time of \LESS{} obtained with $m = 20$ since it is also the default value in 

\begin{landscape}
{
\begin{table}[H]
\scriptsize
  \caption{\scriptsize  Computational results obtained with $5 \times 4$ nested cross-validation. The mean and std of MSE values over the 5 folds are shown.
   The smallest average error for each problem is marked with \textbf{boldface font}.
  }
  \label{tab:results}
  \centering
  {
\begin{tabular}{cc|rrrrrrrrrrrr}
  \hline
  \multicolumn{2}{c}{\textbf{Problem}} 
    & \LESSA{} & \textbf{\LESSB{}} & \textbf{ANFIS} & \textbf{TSK} 
    & \textbf{RF} & \textbf{KNN} & \textbf{MLP} & \textbf{LocR} 
    & \textbf{MAG} & \textbf{GPR} & \textbf{XGB} & \textbf{LGBM} \\
  \hline\hline
  
\texttt{abalone} & Avg. 
    & \textbf{4.36} & 4.42 & 5.30 & 4.66 
    & 4.86 & 4.90 & 4.43 & 4.52 
    & 7.13 & 4.77 & 4.87 & 4.84 \\
(4177×7) & Std. 
    & 0.40 & 0.39 & 0.39 & 0.42 
    & 0.43 & 0.43 & 0.42 & 0.45 
    & 1.48 & 0.43 & 0.45 & 0.41 \\
  \hline

\texttt{airfoil} & Avg. 
    & 6.64 & \textbf{2.59} & 16.36 & 4.98 
    & 3.44 & 8.74 & 4.62 & 5.83 
    & 30.95 & 6.48 & 3.20 & 3.00 \\
(1503×5) & Std. 
    & 0.74 & 0.50 & 0.96 & 0.61 
    & 0.68 & 1.54 & 0.40 & 0.74 
    & 3.02 & 0.70 & 0.76 & 0.74 \\
  \hline

\texttt{cadata ($\times 10^{9}$)} & Avg. 
    & 3.48 & 2.63 & N/A & 2.97 
    & 2.42 & 3.78 & 2.97 & 3.10 
    & 6.35 & 2.88 & 2.25 & \textbf{2.19} \\
(20640×8) & Std. 
    & 0.12 & 0.52 & N/A & 0.08 
    & 0.07 & 0.13 & 0.15 & 0.08 
    & 0.73 & 0.08 & 0.13 & 0.97 \\
  \hline

\texttt{ccpp} & Avg. 
    & 13.81 & \textbf{9.66} & 17.83 & 12.95 
    & 11.56 & 14.78 & 16.93 & 13.72 
    & 23.86 & 15.25 & 10.22 & 9.99 \\
(9568×4) & Std. 
    & 1.14 & 1.26 & 1.33 & 1.23 
    & 1.41 & 1.05 & 1.42 & 1.27 
    & 3.39 & 1.13 & 1.10 & 1.28 \\
  \hline

\texttt{cpusmallscale} & Avg. 
    & 9.44 & 7.10 & N/A & 8.34 
    & 8.91 & 14.56 & 8.67 & 8.15 
    & 278.80 & 77.04 & 7.66 & \textbf{6.84} \\
(8192×12) & Std. 
    & 0.60 & 0.47 & N/A & 0.31 
    & 2.02 & 2.98 & 0.60 & 0.47 
    & 41.16 & 8.36 & 1.50 & 0.23 \\
  \hline

\texttt{energy} & Avg. 
    & 210.45 & 213.79 & N/A & 211.36 
    & 247.88 & 211.86 & 213.08 & 214.33 
    & 211.85 & 277.70 & 210.39 & \textbf{210.28} \\
(19735×26) & Std. 
    & 1.60 & 1.92 & N/A & 1.81 
    & 2.84 & 1.82 & 1.02 & 1.82 
    & 1.84 & 4.51 & 1.71 & 1.65 \\
  \hline

\texttt{housing} & Avg. 
    & 14.16 & 12.63 & 17.28 & 16.59 
    & 12.71 & 24.05 & \textbf{12.31} & 13.71 
    & 60.10 & 26.63 & 16.03 & 14.43 \\
(506×13) & Std. 
    & 8.38 & 7.91 & 6.43 & 6.04 
    & 7.75 & 10.79 & 6.48 & 8.35 
    & 19.80 & 12.27 & 6.57 & 7.01 \\
  \hline

\multicolumn{2}{c}{Avg. ranking} 
    & 6.00 & \textbf{2.86} & 10.50 & 5.57 
    & 5.57 & 9.00 & 5.57 & 6.00 
    & 10.71 & 8.57 & 4.14 & \textbf{2.86} \\
  \hline\hline
\end{tabular}
  }
\end{table}
}
\vspace{-0.5cm}
%
\begin{table}[H]
\scriptsize
  \caption{\scriptsize CIs and p-values for \LESSB{} vs other methods.}
  \label{tbl: confidence intervals}
  \centering
  \resizebox{1\columnwidth}{!}{
\begin{tabular}{c|rrrrrrrrrrr}
  \hline
  \textbf{Diff between \LESSB{}}
    & \LESSA{} & \textbf{ANFIS} & \textbf{TSK} & \textbf{RF}
    & \textbf{KNN} & \textbf{MLP} & \textbf{LocR} & \textbf{MAG}
    & \textbf{GPR} & \textbf{XGB} & \textbf{LGBM} \\
  \hline\hline

CI left
    & -0.0111 & -0.4582 & 0.0223 & -0.0052
    & 0.0841 & -0.0511 & -0.0087 & 0.0191
    & 6.9e-05 & -0.0479 & -0.0598 \\

CI right
    & 0.3455 & 1.8038 & 0.2348 & 0.1385
    & 0.5823 & 0.3065 & 0.2766 & 4.3962
    & 1.0805 & 0.1210 & 0.0850 \\

$p$-value (two-way ANOVA)
    & 0.0615 & 0.1547 & 0.0253 & 0.0638
    & 0.0170 & 0.1312 & 0.0613 & 0.0486
    & 0.0500 & 0.3300 & 0.6846 \\
  \hline

$p$-value (Friedman)
    & 0.1250 & 0.1250 & 0.1250 & 0.0469
    & 0.0312 & 0.1250 & 0.0156 & 0.0312
    & 0.0156 & 0.4688 & 0.9531 \\
  \hline\hline
\end{tabular}
    }
\end{table}
\end{landscape}

{\noindent}our Python implementation. Figure \ref{fig:timings} demonstrates that both versions of \LESS{} are faster than TSK, RF, and GPR and slower than the other methods.
\begin{figure}[H]
  \centerline{
\begin{subfigure}[Problem \texttt{energy} (19,735$\times$26)]{
    \centering
    \includegraphics[width=0.4\linewidth]{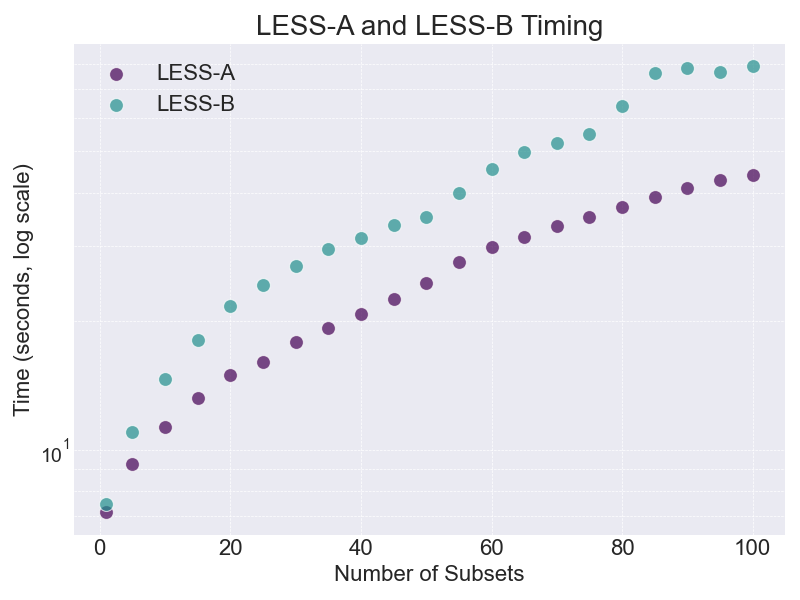}  
    \label{fig:energytime}}
\end{subfigure}
\begin{subfigure}[Problem \texttt{energy} (19,735$\times$26)]{
    \centering
    \includegraphics[width=0.5\linewidth]{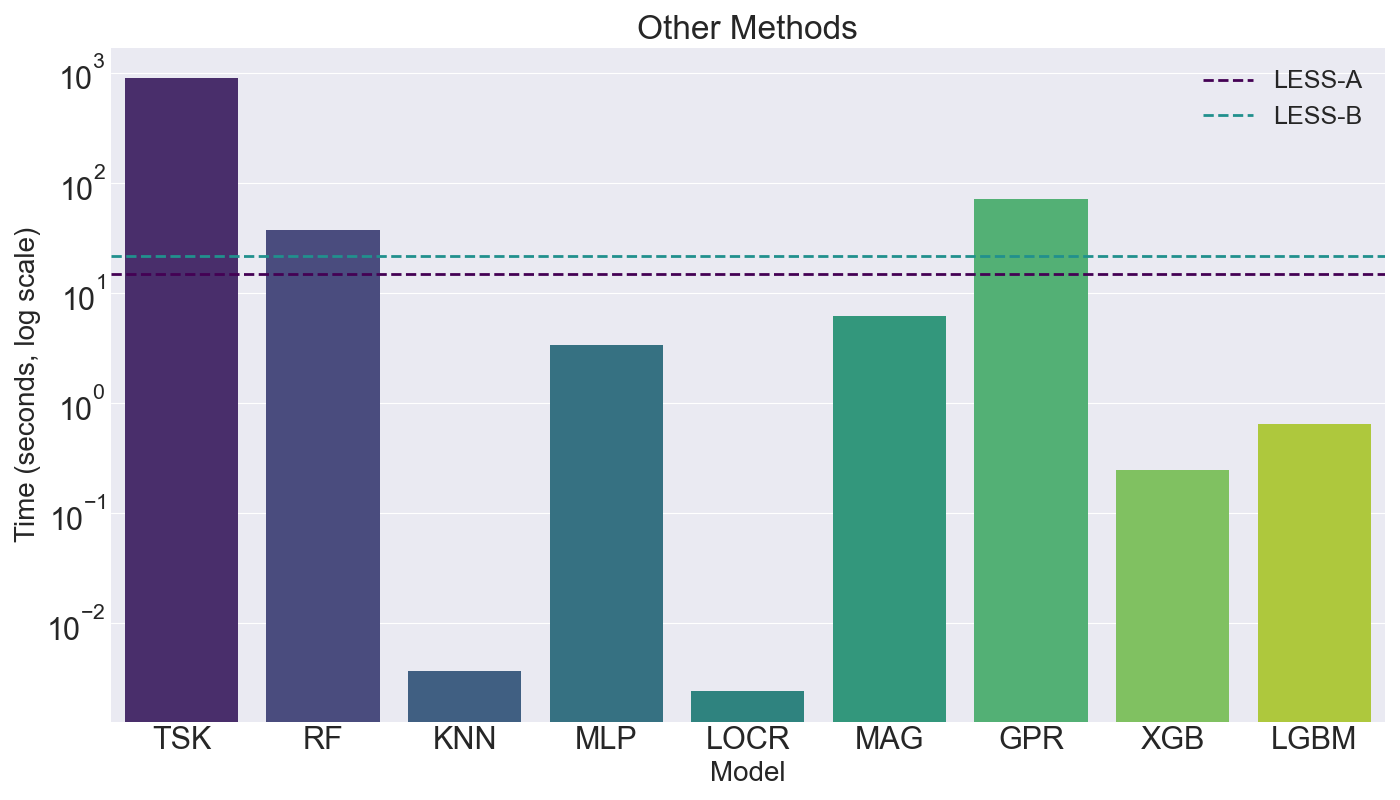} 
    \label{fig:superconductime}}
\end{subfigure}
}
\caption{Computation times of \LESSA{} and \LESSB{} (middle) vs number of subsets. Right: Time comparison against other methods; the horizontal dashed line shows the computation times of \LESSA{} and \LESSB{} with its default value of 5\% of the samples used for setting the number of subsets to 20.}
\label{fig:timings}  
\end{figure}

\section{Conclusion} \label{sec: Conclusion}
In this paper, we have proposed \LESS{}, a generic regression method for learning a model which behaves diversely over the input space. \LESS{} uses both subsampling and stacking ideas in a novel way. The method can also be seen as a two-layer meta-learning algorithm, where local learners feed their outputs as feature vectors to a global learner. Our numerical experiments have shown that \LESS{} is quite competitive when compared to the well-known regression methods. Moreover, \LESS{} is one of the promising methods among the competing methods in terms of computation time. This performance in time can even be improved with a straightforward parallel implementation. The first version of \LESS{} is released as an open-source Python package. 

There are also some limitations of our work as well as some possible extensions that may resolve some of those limitations. 
The performance of \LESS{} is very promising based on our numerical experiments yet there is no theoretical guarantee of its superiority. 
We report results by selecting the subsets randomly. However, it is unclear whether this is the best choice. A deterministic selection of subsets that maximizes the average distance between the subsets may help improve the performance of \LESS{}. Likewise, subset selection can be made via optimization, where the objective would be to obtain the best predictions from the local predictors. These ideas need to be tested.
  
For large datasets, the global learning step in \LESS{} can be time-consuming, especially under specific choices for the global learner that demand the knowledge of the transformed data matrix explicitly. Methods that scale with the data size could be chosen for the global learner to relieve the computational burden. One class of such methods is subsampling-based gradient methods. Also, it may be possible to avoid having to calculate and keep the transformed data matrix while still being able to carry out the calculations. For example, to calculate the ordinary least squares solution in the global learning phase, one does not need to know the transformed data matrix; instead, one can calculate the required matrix multiplications by sequentially processing its columns.

Although in principle \LESS{} can also be applied to classification, it is not clear how the categorical output values should be represented in the transformed data matrix. To this end, we are currently trying different approaches and testing them on standard data sets.

\section*{Acknowledgments}

We thank Kamer Kaya from Sabancı University for his help with the implementation of \LESS{}.





\footnotesize{
{\parindent0pt
\bibliography{LESS}
}}

\clearpage
\appendix

{{\noindent}\large{\textbf{Appendices}}}



\section{Related Methods} \label{sec: Related Methods}
There are several approaches closely related to \LESS{} in the literature. In this section, we discuss the most important ones among them and their relations to \LESS{}. Before, we start with a remark on the relation of \LESS{} to kernel approach, which is used by many learning methods discueed here. 

\begin{remark}[Relation to kernel methods] The kernel approach is used in many learning methods, such as support vector machines, kernel ridge, and Gaussian process regression. Considering \eqref{eqn:xhat}, we can define for any two vectors $\vx_s$ and $\vx_t$, the kernel as
\begin{align*}
  \kappa(\vx_s, \vx_t) &= \vz_{s}\tr \vz_{t} = \sum_{j=1}^m w_{j}(\vx_s) \CL(\vx_s | \mX_j, \vy_j)\CL(\vx_t | \mX_j, \vy_j) w_{j}(\vx_t).
\end{align*}
If the fitted local models are linear regressors, with each $\CL(\vx | \mX_j, \vy_j) = \vv_{j}\tr \vx$ for some $\vv_{j} \in \mathbb{R}^{p}$, then by using our notation, the kernel simplifies to
\[
\kappa(\vx_s, \vx_t) = \vx_s\tr \mV(\vx_s) \mV(\vx_t)\tr \vx_t,
\]
where for any $\vx \in \RR^p$, the $p \times m$ matrix $\mV(\vx)$ is defined such that its column $j$ is simply $\vv_{j} w_{j}(\vx)$. The derivation above also shows that if we use a global learning method that works with kernels, then we may not need to store the $\mZ$ matrix explicitly.
\end{remark}

\textbf{Stacking.} \LESS{} is related to the well-known method of stacking \citep{Breiman_1996b} in the light of equation \eqref{eq: global learner as linear combination}. Stacking also learns a linear combination of individual predictions to have a global prediction. However, different than stacking, \LESS{} involves a distance-based weighting $w_{j}(\vx_0)$, and instead of learning the overall coefficient $w_{j}(\vx_0) \beta_{j}$ for the $j$'th predictor, it learns $\beta_{j}$ only. The weight $w_{j}(\vx_0)$ stems from the way \LESS{} forms subsets on which it obtains its local predictors, which is another feature of \LESS{} that makes it different from the stacking methods. In the default implementation of \LESS{}, the subsets are formed by first selecting a sample at random as an anchor, followed by the selection of the $k$-nearest samples to that anchor. With such subsets that are concentrated at different locations, it becomes relevant to address distances from and to a subset. \LESS{} uses those distances to take into account the possible heterogeneity in the model: When we predict at a given point, \LESS{} penalizes the prediction obtained from a subset by the distance between the point and the \textit{centroid} of the subset. 

We would like to point out the generality offered by \LESS{} compared to stacking. The \LESS{} method described in the example above uses the ordinary least squares method for its global learner $\mathcal{G}(\vz | \mZ, \vy)$. However, in general, this global learner may be chosen to be any appropriate learning method.

\textbf{Magging.} The magging method proposed by \cite{BuhlmannMeinshausen2015Magging} was mentioned in Section \ref{sec: Introduction}. Magging also respects the heterogeneity in data. Magging proposes to use a convex combination of multiple estimators, found based on subsets of data as its final estimator. The coefficients of the convex combination are those that minimize the norm of the vector of final predictions. Since magging is presented in the context of linear models, this weighted combination can be transformed into a weighted combination of predictions at any point as the final prediction. However, there are differences between magging and \LESS{}. First, in magging, the same linear combination of the subset predictions is used everywhere for prediction. In contrast, in \LESS{}, the subset predictions are first scaled based on the distance between the center of their subsets to the input point, and only then those scaled predictions are transformed with a common linear combination to yield the final estimate. Second, magging uses non-local subsets while \LESS{} uses local subsets to compute the first stage predictions before aggregation. Furthermore, magging is based on the idea that any common effect across all subsets will be retained by the optimum convex combination. \LESS{} is arguably more flexible in the sense that it also accommodates heterogeneity caused by different components of a point being active in different locations in the input space.

\textbf{Local learning.}
Local learning methods are also closely related to \LESS{}. The LOcally Estimated Scatterplot Smoothing (LOESS) method by \cite{Cleveland_1979} fits a weighted regression model around each data point. This is a generalization over kernel regression, also known as the Nadaraya-Watson estimator by \cite{Nadaraya_1964}. Methods like LOESS, which \textit{delay} their learning until a query is made to the system, are also called ``lazy learning'' methods. Our implementation of local linear regression (LocR), used for benchmarking in Section \ref{sec: Comparison Against Other Methods}, is a variant of LOESS.

Another related model is the Gaussian process regression (GPR) model \citep{Rasmussen_and_Williams_2006}. In GPR, a Gaussian process is used to address the distance among the samples during prediction. A covariance matrix of these distances takes higher values when two samples are close to each other. Then, based on that covariance matrix, the prediction function is \textit{enforced} to take closer values at points in close vicinity of each other. When the size of the dataset is too large to work out the covariance matrix efficiently, subsampling ideas can be used. It is common to resort to subsets \cite[Chapter 8]{Rasmussen_and_Williams_2006}, either chosen greedily or randomly, such as the subset of regressors, subset of data points, and projected process methods. Among those approximate methods, the Bayesian Committee Machine \citep{Tresp_2000} and its variants \citep{Liu_et_al_2018} are examples of aggregation models that are worth mentioning.

Like LOESS and GPR, \LESS{} also tends to rely on the neighbor samples when it performs prediction. This is done via the weighting functions $w_{j}(\cdot)$, which reduces the effect of a local learner on a distant sample. However, unlike LOESS and GPR, \LESS{} also allows \textit{distant} local learners to have a larger effect on the prediction by learning the coefficients $\beta_{j}$, $j = 1, \ldots, m$, in its global learning phase. There is an exception to the way locality is enforced regarding GPR: The covariance matrix of GPR can be chosen to be based on a trigonometric function of distance, therefore capturing a repeated pattern in the model. However, the periodicity assumption can be restrictive, let alone the tuning challenges.

Scalable and local versions of Gaussian process regressors, namely the infinite mixture of Gaussian process experts, have been proposed by \cite{Rasmussen_and_Ghahramani_2001, Meeds_and_Osindero_2005}. In this model, each data point is assigned to a cluster 
The prior for this assignment is constructed with a Dirichlet process, and hence, the number of clusters is unbounded. A separate Gaussian process governs each cluster. Therefore, the cost of this algorithm is cubic in the number of inputs assigned to the largest cluster (which can be limited by design) as opposed to the cost which is cubic in the number of data points in the entire dataset. The idea of clusters is to explore the locality, and also, to benefit from computational savings. As the assignments are latent, they need to be estimated along with other parameters using Markov Chain Monte Carlo (MCMC) algorithm. The initial idea belongs to \cite{Rasmussen_and_Ghahramani_2001}, while \cite{Meeds_and_Osindero_2005} proposed a variant with the property of being a generative model for the inputs. For classification, a Bayesian way of the cluster-and-predict method, also based on Dirichlet processes, is given by \cite{Shahbaba_and_Neal_2009}.

Although the infinite mixture of Gaussian process experts has the advantage of allowing any $\#$ clusters, it is more costly compared to \LESS{}, since the MCMC algorithm designed for it needs to sweep over all the points in the dataset in every iteration. The mixture of experts and sparsity approaches are combined by \cite{Shazeer_et_al_2017}, who propose the Sparsely-Gated Mixture-of-Experts Layer as a component of a deep learning network. Here, for an input sample, a small number of experts among many are activated by a gating network. The method, although generic, is applied for language modeling and machine translation where the Mixture-of-Experts is applied between layers of a recurrent neural network.

The locally linear ensemble method for regression \citep{Kang_and_Kang_2018} performs an expectation-maximization algorithm to form clusters and train simple linear models for each cluster. Prediction is then performed as a weighted average of the local predictions. While the method is simpler, it is more specific compared to \LESS{} in terms of the range of local models and the way the local models are combined. Moreover, while \LESS{} is non-iterative, the method of \citet{Kang_and_Kang_2018} is iterative and it needs to sweep over all the points in the dataset in every iteration, similar to the Gaussian process experts approach.

\textbf{Fuzzy inference.}  Fuzzy inference systems (FIS) \citep{Babuska_1998, Mamdani_1974, Takagi_Sugeno_1985_FIS} are also related to \LESS{}. In particular, the Takagi-Sugeno (TS) fuzzy models \citep{Takagi_Sugeno_1985_FIS} and the adaptive neuro-fuzzy generalization of FIS, namely ANFIS \citep{Jang1993ANFIS}, are two successful frameworks. Several modern implementation tools exist for those frameworks; see, e.g.\ \citet{Cui_et_al_2022, Zhang_and_Chen_2024} for TS and ANFIS, respectively.


ANFIS boils down to TS fuzzy models when linear models are used in the clusters to predict the output. In that respect, \LESS{} with linear local and global models has similarities with TS fuzzy models as well \citep{YenWang1998ImprovingTSK}.  That being said, \LESS{} and ANFIS are essentially different both on a technical level and in terms of the approach in the design and combination of clusters. Below, we outline some of the important differences.

\noindent \textit{(i) Optimized parameters and computational load:} One difference between \LESS{} and ANFIS is in the number of parameters to learn. In ANFIS, both parameters of the membership functions (corresponding to the weights in \LESS{}) and the consequent parameters (corresponding to the local model parameters in \LESS{}) are learned during training. The training is an iterative procedure, where a back-propagation step is used to optimize the weights and the least-squares method is used to update the consequent parameters. In contrast, \LESS{} does not optimize its distance-based weights; the weight functions are fixed. 
A related difference is in the way the consequent parameters of ANFIS, which correspond to the local model parameters in \LESS{}, are learned. While ANFIS learns its consequent parameters simultaneously for all its clusters, \LESS{} learns them separately for each cluster in isolation from the other clusters. 
These differences suggest that, while ANFIS performs a finer tuning in training, \LESS{} is a lot easier to train with significantly reduced computational complexity.

\noindent 
\textit{(ii) Combining the `local' outputs:} In ANFIS, the final output is often a weighted average of the linear outputs obtained from the subsets. \LESS{} is somehow more general in combining those `local' outputs: Once the local models are trained, the final output prediction is fit over a projected space where the original data is transformed using weights $w_{j}(\mathbf{x})$. The mentioned generality is because this global learner can be a nonlinear regressor, such as a decision tree.


\textbf{Feature learning.} \LESS{} is a two-stage algorithm where the first stage somewhat involves representing input variables as feature vectors of predictors. In that sense, \LESS{} can be related to representation/feature learning methods; see, \cite{Bengio_et_al_2013, Le-Khac_et_al_2020} for comprehensive reviews. There are many methods for representation learning, such as those based on sparse representations \citep{Lee_et_al_2006} and neural networks \citep{Hinton_and_Salakhutdinov_2006}. \LESS{} is especially related to those feature learning methods based on techniques that respect the locality of input, such as those involving $k$-means clustering \citep{Coates_and_Ng_2012}, vector quantization, locally linear embedding \citep{Roweis_and_Saul_2000}, local coordinate coding \citep{Yu_et_al_2009, Yu_and_Zhang_2010}, feature clustering \citep{Xu_and_Lee_2015}. Although those methods are also motivated by dimensionality reduction, their relation to \LESS{} is mostly due to their acknowledgment of the non-homogeneity of input data. Moreover, feature learning methods \citep{Xu_and_Lee_2015} are directly motivated by the use of dimensionality reduction for regression. \LESS{} is not necessarily a dimension-reduction method, since the generated feature vectors can be even larger in dimension than the original inputs. A major difference between \LESS{} and the mentioned works in the way of generating feature vectors is that while the mentioned works are unsupervised, \LESS{} uses both input and output to learn its features. An exception to this is the work of \cite{Parelta_and_Soto_2014}, where local feature selection is performed in the clusters of a mixture-of-experts model for classification.

\textbf{Subspace methods.} Another line of related research concerns selecting random subsets of input vector components, also known as feature selection. \cite{BootNibbering2019RandomSubspace} present asymptotic properties and upper bounds for random subspace methods using random feature selection and random Gaussian weighting of predictors. The dimension of the feature space is reduced by selecting random features. This approach can be counted as another feature learning approach, which is fundamentally different than \LESS{} where the source of randomness is the random sample subset selection, not feature selection.

\section{Choices for Subset Selection and Dataset Splitting} \label{sec: Variants} 

The first choice is to apply a clustering method in place of subset selection. Some of the clustering methods require setting the number of clusters in advance, like $k$-means, and others, like $X$-means clustering \citep{Pelleg_and_Moore_2000}, determine the number of clusters automatically. In either case, these clusters constitute the subsets in \LESS{}. Unlike selecting the subsets with anchor points, clustering methods entail obtaining subsets with varying sizes and they partition the dataset into mutually exclusive, collectively exhaustive subsets. Note that if the clustering method is not stochastic, or in other words, it does not produce different clusters by varying random seeds, then the averaging step can safely be omitted.

The second choice is based on splitting the training dataset into two parts. The first part is used for local learning, whereas the second part is reserved for global learning. This can be considered a validation approach since the local predictors and the global predictor are trained with independent parts of the dataset. We expect that this approach may decrease the variance of the resulting \LESS{} variant by reducing the potential to overfit to the same dataset used for both local and global learning. However, there is also a trade-off because splitting the dataset into two parts decreases the sizes of the training datasets. With such smaller training datasets, the performances of the local and global predictors may deteriorate. We also conduct numerical experiments with these variants in the next section. 

Table \ref{tbl: choices} summarizes the structural choices for \LESS{} that have been discussed so far. As the table implies, there are several structural choices for \LESS{}, as well as several numerical parameters such as the number of subsets, $m$, and the sample size $k$. As done in Remark \ref{rem: default choices for LESS}, it is possible to declare some default choices for \LESS{} and tune for the others with cross-validation.

\begin{table}[H]
\caption{Structural choices for \LESS{}}
\label{tbl: choices}
\centerline{
\begin{tabular}{l c}
\toprule
 Structural choices & Alternatives \\
\midrule
Subset selection & random anchors, clustering \\
Data for local and global phases & split, do not split \\
Local model & any regression model \\
Global model & any regression model \\
Use of repetitions & averaging, boosting \\
\bottomrule
\end{tabular}
}
\end{table}

\section{Proofs of Propositions}\label{appendix::proofs}

\begin{proof}[Proof of Proposition~\ref{proposition::OLS}]
When $\lambda = 0$, all predictions from the subsets will be weighted equally with $w_{j}(\vx) = 1/m$ for all $j = 1, \ldots, m$.  Let $\mV$ be the $p \times m$ matrix whose $j$'th column is $\vv_{j}$. Then, with equal weights, we have $\vz_{i} =  \mV\tr \vx_{i} /m$ and $\mZ = \mX \mV /m$. The global model finds a $\vbeta$ such that $\| \mZ \vbeta - \vy \|_{2}$ is minimized. On the other hand, consider minimizing $\| \mX \vu - \vy \|_{2}$ with respect to $\vu$, and let $\vu^{\ast}$ be a solution. Since $\mV$ has rank $p$, the column space of $\mZ$ is spanned by the columns of $\mX$, and as a result there exists a $\vbeta$ which minimizes $\| \mZ \vbeta - \vy \|_{2}$ and satisfies $\mV \vbeta / m = \vu^{\ast}$. Therefore, the prediction at $\vx_{0}$ is $\vbeta\tr \mV\tr \vx / m = \vu^{\ast \intercal} \vx_{0}$. This concludes that \LESS{} reduces to OLS.
\end{proof}

\begin{proof}[Proof of Proposition~\ref{proposition::lambda-infty}]
Given the subsets, let $c_{i}$ be the number of the subset whose center is closest to the data point $\vx_{i}$. Then, as $\lambda \rightarrow \infty$, the feature vector $\vz_{i}$ has for $j=1, \dots, m$, the components 
\[
z_{ij} = \begin{cases} \vx_{i}\tr \vv_{j}, & \text{ for } j = c_{i}; \\
0, & \text{ otherwise}.
\end{cases}
\]
This means that the subsets do not learn from each other, and for each $j \in \{1, \ldots, m\}$,  \eqref{eq: linear global beta} reduces to finding the component $\beta_{j}$ that minimizes 
\[
\sum_{i: c_{i} = j}  (\vx_{i}\tr \vv_{j} \beta_{j} - y_{i})^{2}.
\]
(This is solved at $\beta_{j} = \sum_{i: c_{i} = j} y_{i} /\sum_{i: c_{i} = j} \vx_{i}\tr \vv_{j}$.) Moreover, due to the assumption in the proposition, subset $j$ contains exactly all of those $\vx_{i}$'s with $c_{i} = j$ . Thus, the minimization problem above for $\beta_{j}$ becomes
\[
\min_{\beta_{j}} \| \mX_{j} (\mX_{j}\tr \mX_{j})^{-1} \mX_{j}\tr \vy_{j} \beta_{j}- \vy_{j} \|_{2}^{2}
\]
which is solved at $\beta_{j} = 1$. The leads to the fact that the global model predicts the response at a test point $\vx_{0}$ as $\vv_{j}\tr \vx_{0}$, the prediction of the local model of the closest subset to $\vx_{0}$.
\end{proof}

\section{Further Computational Results} \label{appendix::results}

In this section we first provide the hyperparameters tested during our comparison and the detailed ranking performance of each method. Then, we present results concerning the effect of using different local and global learners. Lastly, we give an analysis on the effect of using different subset selection data splitting strategies.

\subsection{Comparison Against Other Methods}\label{appendix::comp-against-others}

\begin{table}[H]
    \caption{Hyperparameter sets used for all methods. The default values are \underline{underlined.}}
    \label{tab:hyps}
    \centering
    \resizebox{1\columnwidth}{!}{
    \begin{tabular}{l|l}
    \hline
    \multicolumn{1}{c}{\textbf{Method}} & \multicolumn{1}{c}{\textbf{Hyperparameter Sets}} \\
    \hline\hline
    \LESSA{}, \LESSB{} & 
    $\#$ subsets, $\{10, \underline{20}\}$; kernel coef., $\{0.1, \underline{0.01}\}$; local model, $\{\underline{\textit{linear}}, \textit{tree}\}$ \\
    \hline
    RF & $\#$ estimators, $\{\underline{100}, 200\}$ \\
    KNN & $\#$ neighbors, $\{3, \underline{5}, 10, 20, 50, 100\}$ \\
    MLP$^*$ & regularization parameter $\alpha$, $\{\underline{0.0001}, 0.001, 0.01\}$ \\
    LocR & percentage of samples, $\{1\%, 5\%, 10\%, \underline{20\%}\}$ \\
    MAG & percentage of samples, $\{10\%, \underline{20\%}, 30\%\}$ \\
    GPR & constant for the kernel matrix $\alpha$, $\{\underline{0.1}, 0.0001, 0.00000001\}$ \\
    LGBM & 
    learning rate, $\{0.01, \underline{0.1}\}$; $\#$ estimators, $\{\underline{100}, 200\}$; $\#$ leaves, $\{31, 63\}$ \\
    XGB & 
    learning rate, $\{0.01, \underline{0.1}\}$; $\#$ estimators, $\{\underline{100}, 200\}$; maximum depth, $\{6, 8\}$ \\
    ANFIS & 
    hybrid, $\{\underline{\textit{True}}, \textit{False}\}$; $\#$ epochs, $\{25, 50\}$ \\
    TSK & 
    fuzzy index, $\{\underline{2.0}, \textit{auto}\}$; $\#$ clusters, $\{5, 10, 20\}$; consequent, $\{\textit{Ridge}, \textit{Decision Tree}, \textit{RF}\}$ \\
    \hline
    \multicolumn{2}{l}{\small $^*$maximum $\#$ iterations is set to 1{,}000 to ensure convergence.} \\
    \end{tabular}}
\end{table}

\begin{figure}[H]
\includegraphics[scale = 0.5]{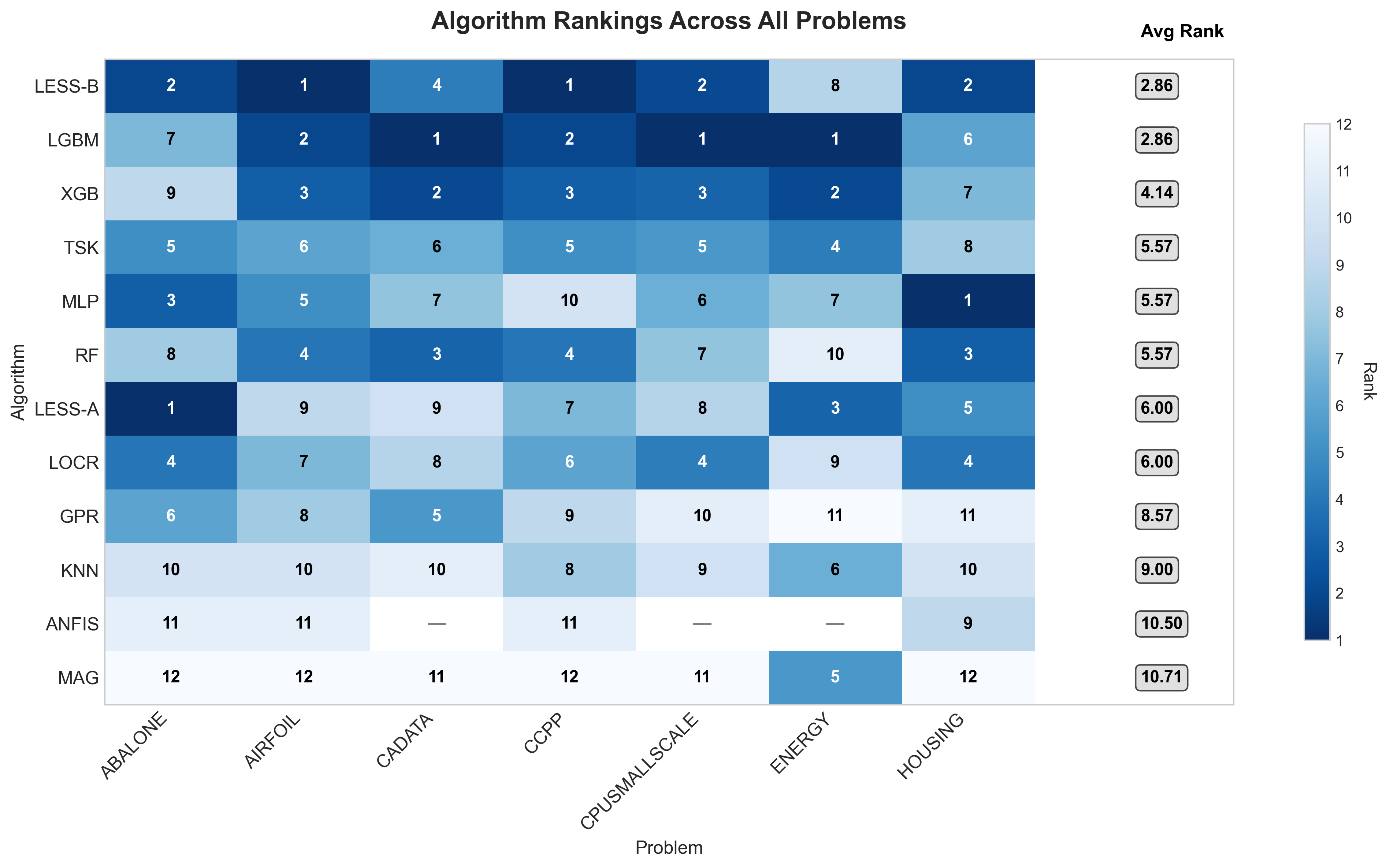}
\caption{Algorithm rankings. For each dataset $i$ and method $j$, $R_{i,j}$ is displayed in the corresponding cell. The methods are ordered from top to bottom with respect to their average rankings $\frac{1}{7} R_{i, j}$, shown in the right-most column.}
\label{fig: Algorithm Rankings}
\end{figure}

\subsection{Effect of Different Local and Global Learners} \label{sec: Effect of Difference Local and Global Learners}
In this experiment, we looked at the performance of \LESS{} with different combinations of choices of method for local learning and global learning phases. For local learning, we considered linear regression (L) and the decision tree (DT) methods. For the global learning phase, we considered linear regression (L) vs random forest (RF). Figure \ref{fig:incr} compares the performance of the resulting four combinations of ``local-global'' learning methods, labeled as \texttt{L-L}, \texttt{DT-L}, \texttt{L-RF}, and \texttt{DT-RF}.
The figure demonstrates that the performance of \LESS{} can be improved when nonlinear estimators are used for local and global learning. Indeed, except for the \texttt{abalone} and \texttt{energy} datasets, \LESS{} with local DT and global RF estimators yields a better performance than its all-linear variant. For the majority of the datasets, using at least one non-linear method in both local and global models yields the best results. 

\begin{figure}[H] 
 \centerline{
 \includegraphics[width = 0.7\linewidth]{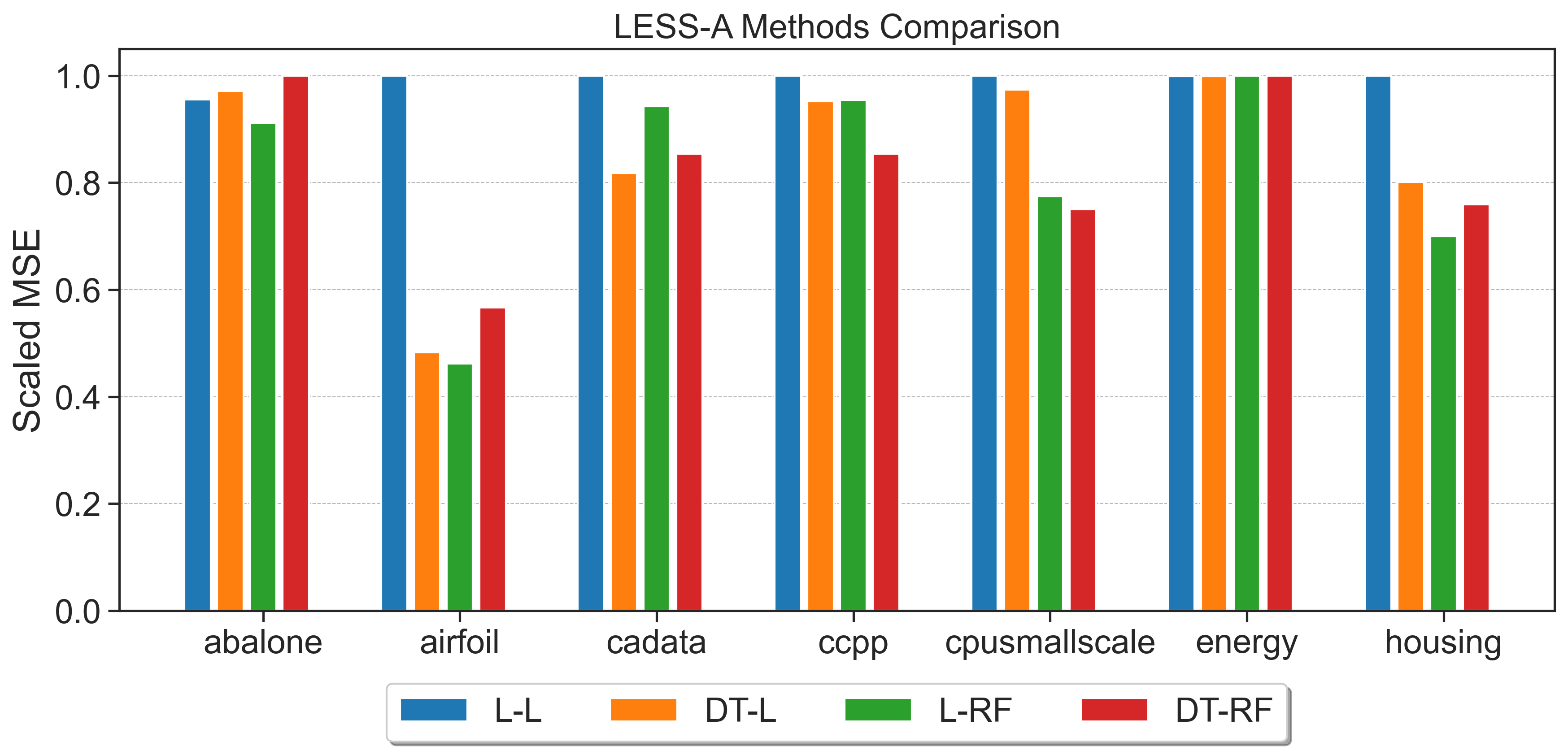}
 }
 \centerline{
 \includegraphics[width = 0.7\linewidth]{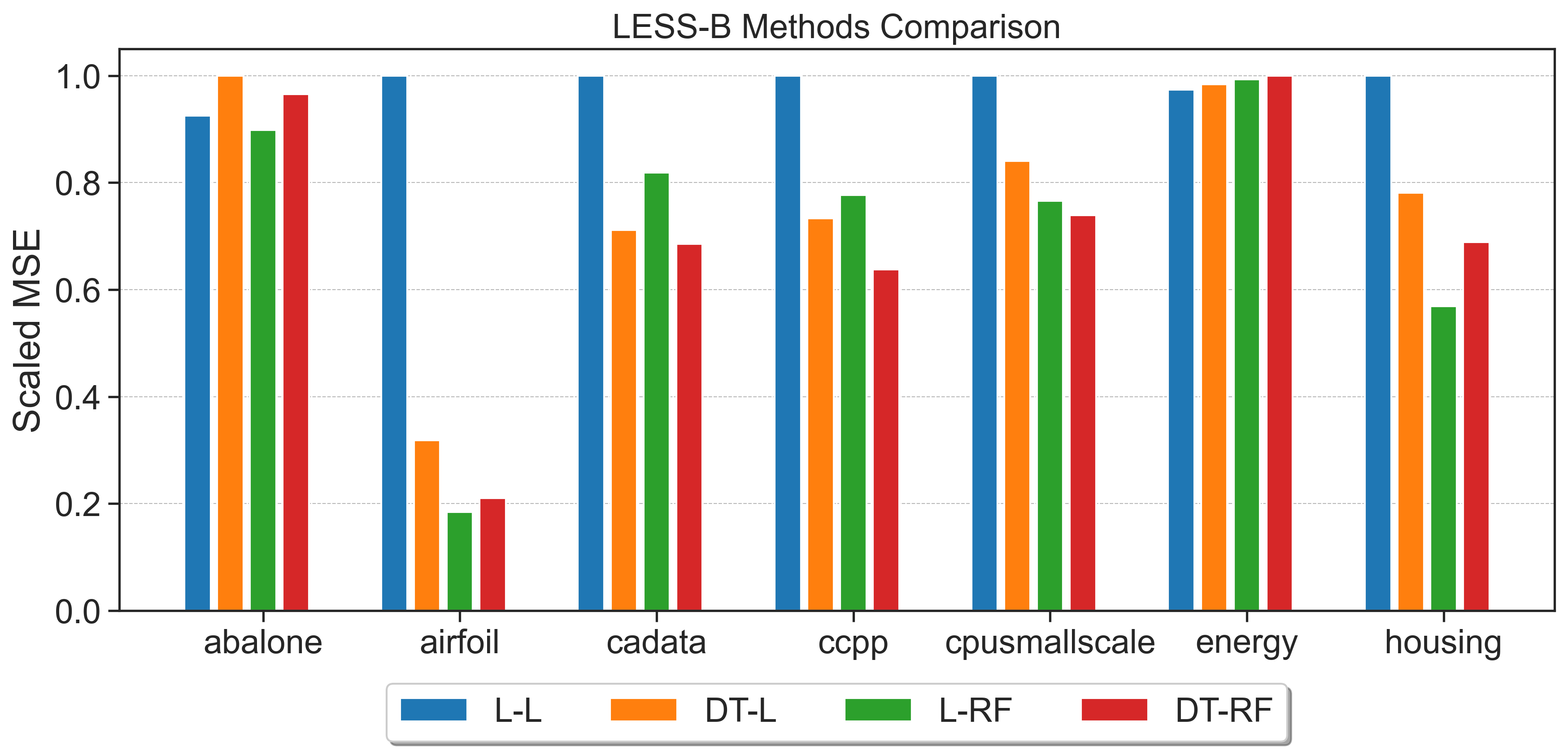}
 }
\caption{Results of changing local and global estimators in \LESS{}. Vertical axis is obtained by scaling the MSE values with the maximum. Options for local estimators are \texttt{L}: linear regression, \texttt{T}: decision tree. Options for global estimators are \texttt{L:} linear regression, \texttt{RF}: Random forest.}
 \label{fig:incr}
\end{figure}

\subsection{Effect of Choices for Subset Selection and Data Splitting} \label{sec: Performances of Variants}

We have conducted our last comparative experiments on variants of \LESS{} obtained with different choices for subset selection and dataset splitting, as described in Section \ref{sec: Variants}. While the default choice for subset selection is by randomly selecting $m$ anchors, we have considered the $k$-means clustering method as an alternative. To test the validation approach, we use a 70\%-30\% split of the dataset for training the local predictors and the global predictor, respectively. Consequently, we obtain four alternative variants of each of \LESSA{} and \LESSB{} that uses the default choices given in Remark \ref{rem: default choices for LESS}: \texttt{NoV-NoC} that uses neither validation set nor clustering, \texttt{NoV-C} that uses only clustering, \texttt{NoC-V} that uses only validation set, and finally \texttt{V-C} that uses both validation set and clustering. For all four methods, we used linear regression for the local models and random forest for the global model. As before, we have conducted nested cross-validation for performance measurement and hyperparameter tuning. We have used the same parameter grids given in Section \ref{sec: Comparison Against Other Methods}, except for $m$, which we tuned from the set $\{5, 10, 20, 100\}$ for all the methods. Note that for \texttt{NoV-C} and \texttt{V-C}, $m$ is the number of clusters.

Figure \ref{fig:vars} shows the results obtained with different variants. The results indicate that those variants have similar performances on most datasets, although the differences are more visible for the \texttt{airfoil} and \texttt{housing} datasets. Among those datasets, for \texttt{airfoil}, both \LESSA{} and \LESSB{} consistently benefited from using validation sets. For the \texttt{housing} dataset, however, the effects of the variants are not similar between \LESSA{} and \LESSB{}. On the whole, the performance of the \texttt{LESS} methodology in general seems robust against the choices concerning using a validation set and clustering.


\begin{figure}[H]
   \centerline{
 \includegraphics[width = 0.7\linewidth]{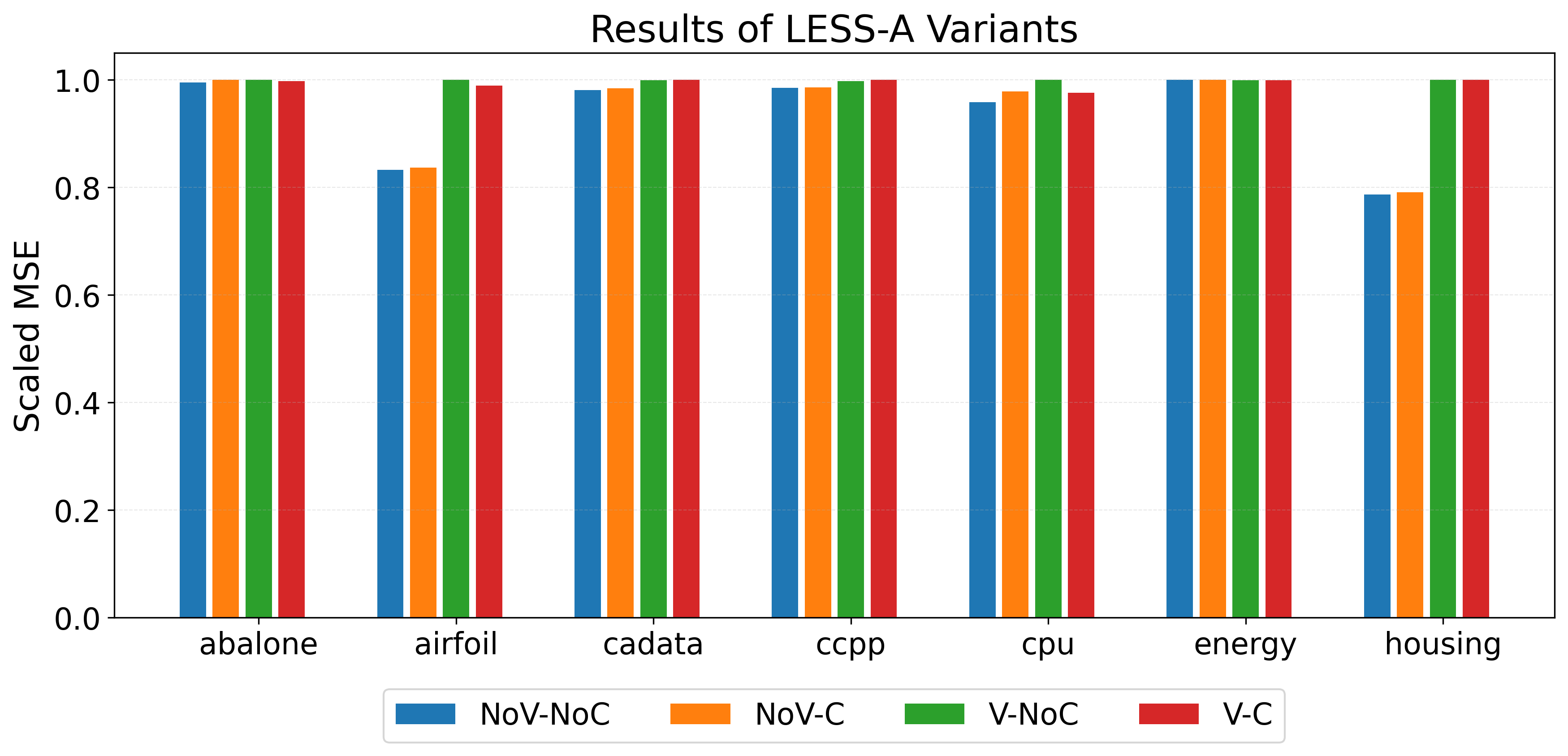}
 }
  \centerline{
 \includegraphics[width = 0.7\linewidth]{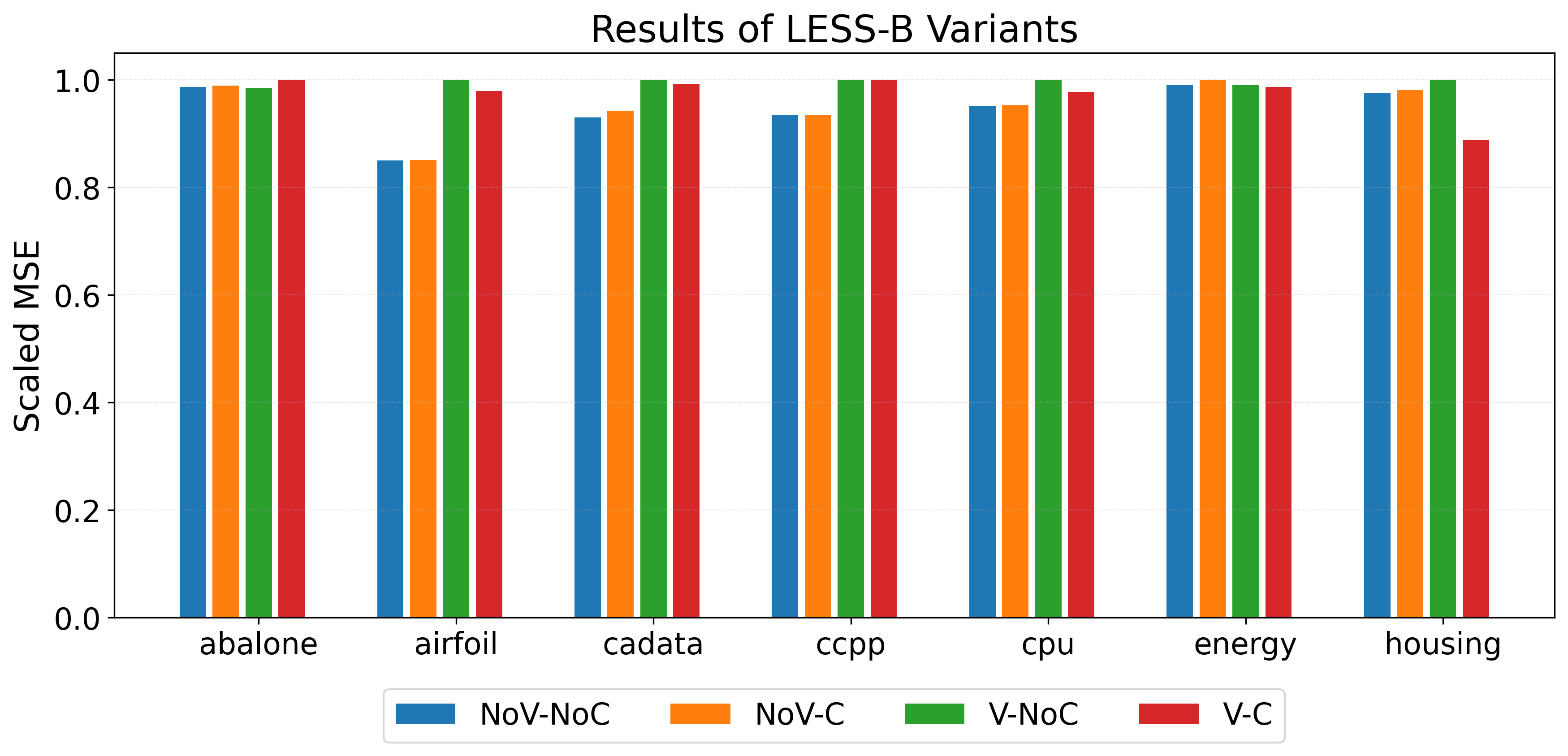}
 }

  \caption{Results of \LESS{} variants. Vertical axis is obtained by scaling the MSE values with the maximum one among \texttt{V-C} (validation, clustering),  \texttt{NoV-C} (clustering, no validation), \texttt{NoC-V} (no clustering, validation),  \texttt{NoV-NoC} (default)}
  \label{fig:vars}
\end{figure}

\end{document}